\documentclass{article}

\PassOptionsToPackage{numbers, compress}{natbib}

\usepackage[main,final]{neurips_2026}

\usepackage[utf8]{inputenc} 
\usepackage[T1]{fontenc}    
\usepackage{hyperref}       
\usepackage{url}            
\usepackage{booktabs}       
\usepackage{amsfonts}       
\usepackage{nicefrac}       
\usepackage{microtype}      
\usepackage{xcolor}         

\usepackage{amsmath,amssymb,amsthm}
\usepackage{graphicx}
\usepackage{wrapfig}
\usepackage{float}
\usepackage{subcaption}
\usepackage{etoolbox}
\usepackage{algpseudocode}          
\floatstyle{ruled}                  
\newfloat{algorithm}{htbp}{loa}
\floatname{algorithm}{Algorithm}

\newtheorem{corollary}{Corollary}
\newtheorem{proposition}{Proposition}
\newtheorem{assumption}{Assumption}
\newtheorem{definition}{Definition}

\title{When Attribution Patching Lies:\\Diagnosis and a Second-Order Correction}

\author{%
  Luyang Zhang \\
  Carnegie Mellon University \\
  \texttt{luyangz@andrew.cmu.edu} \\
  \And
  Jialu Wang \\
  University of California, Santa Cruz \\
  \texttt{faldict@ucsc.edu} \\
}

\begin{document}

\maketitle

\begin{abstract}
A central goal of mechanistic interpretability is to identify which internal components causally drive a language model's behavior. Because these importance estimates serve as the evidence for identifying circuits, systematic errors can lead to the misidentification of the underlying mechanisms. While activation patching provides a gold-standard causal metric, its computational cost is prohibitive at scale. Practitioners instead rely on attribution patching, a gradient-based, first-order approximation whose reliability remains poorly understood. In this work, we characterize the source of this unreliability, demonstrating that the dominant error stems from the non-linearities in the downstream network rather than local curvature at the patched component. This insight yields three practical tools: (i) a reliability score to detect untrustworthy estimates, (ii) error bounds quantifying potential attribution mis-specifications, and (iii) a Hessian-vector-product (HVP) correction that eliminates the leading-order error with only one additional backward pass. In evaluations across five model families (124M--9B parameters) and both random-token and naturalistic (name-swap) perturbations, HVP is the only second-order correction feasible at larger scale, where standard baselines like Integrated Gradients become computationally prohibitive. In comparative experiments, a multi-step HVP variant matches or exceeds the accuracy of Integrated Gradients at significantly lower compute, outperforming prior second-order baselines. These improvements lead to higher-fidelity circuit recovery on standard benchmarks and support a \emph{Screen-Flag-Fix} workflow that targets computational effort only toward the components flagged as unreliable.
\end{abstract}

\section{Introduction}
\label{sec:intro}

As language models grow in scale and capability, understanding their internal mechanisms becomes increasingly important.
Mechanistic interpretability seeks to provide this understanding by explaining model behavior in terms of internal computations. A central step in that agenda is causal localization: identifying which attention heads, neurons, or features causally drive a given behavior. These localization scores are often used to support circuit claims~\citep{conmy2023acdc, marks2025sparse, anthropic2025circuits}, so systematic error can lead to incorrect mechanistic conclusions. The gold-standard causal test is \emph{activation patching}~\citep{elhage2021mathematical, wang2023ioi}, which replaces a component's activation under a corrupted input with the clean-input value and measures the output change. But its cost scales linearly with the number of components, quickly becoming prohibitive in modern Large Language Models (LLMs). In practice, broad localization therefore relies on cheaper approximations, with direct interventions reserved for a short list of components.

A commonly used approximation is \emph{attribution patching}~\citep{nanda2023ap}, which replaces many explicit interventions with a single backward pass by linearizing the effect around a corrupted activation. However, this linearization can be unreliable when downstream nonlinearities are strong. Prior work has identified concrete failure modes and proposed partial remedies~\citep{kramar2024atp, edin2025gim, relp2025}, yet a fundamental question remains: for a given component, when is attribution patching reliable, how large can its error be, and how should it be corrected?

\begin{figure}[t]
\centering
\includegraphics[width=\linewidth]{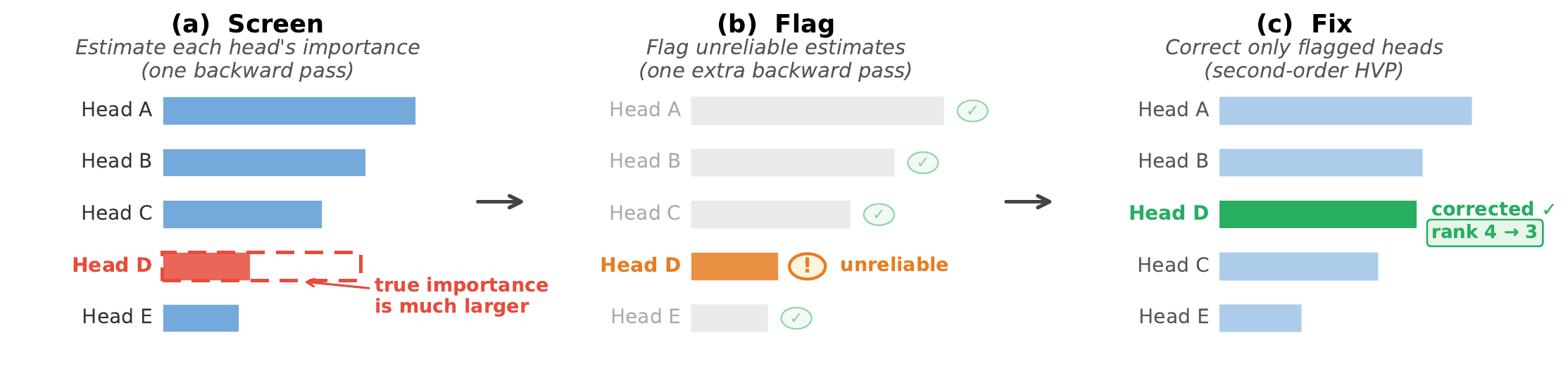}
\caption{\textbf{Screen--Flag--Fix pipeline} for reliable attribution patching. (a)
Attribution patching screens all heads cheaply; (b) a reliability score flags
suspect estimates; (c) HVP corrects only the flagged heads, recovering the
true ranking.}
\label{fig:framework}
\end{figure}

We address this gap by analyzing the structure of attribution patching error. The key finding is that the dominant error stems from the network's response to the intervention, rather than from the local nonlinearity at the patched component. A second-order Taylor expansion makes this precise: the error splits into a leading quadratic term, computable from a single Hessian-vector product (HVP), plus a cubic remainder. This decomposition yields three key results: (1) a reliability score for flagging unreliable estimates; (2) error bounds with a provable $1/K^2$ convergence rate, where $K$ is the number of sub-steps a correction is split into; and (3) an explanation for the insufficiency of prior fixes~\citep{kramar2024atp, edin2025gim}. Specifically, while prior methods use local curvature, we show that the true error depends on the downstream network response, a quantity that differs from local metrics by $2$--$66\times$ with a near-zero correlation.

Building on this analysis, we propose a \emph{Screen--Flag--Fix} pipeline (Figure~\ref{fig:framework}): screen all components with attribution patching, flag unreliable ones via the reliability score, and correct only those with an HVP-based second-order fix. For large perturbations where the single-step expansion overshoots, a multi-step variant (MS-HVP) splits the correction into $K$ sub-steps, each evaluated locally. We evaluate across five model families (124M--9B parameters), comparing against Integrated Gradients (IG)~\citep{sundararajan2017ig}, Integrated Hessians (IH)~\citep{janizek2021hessians}, and GIM~\citep{edin2025gim}. \emph{At larger scale} (8B+ parameters), IG requires ${\sim}25$ GPU-days per task; HVP is the only second-order correction demonstrated at this scale, reducing error by up to $82\%$ (MS-HVP $K{=}5$ on Llama-3.1-8B). \emph{At smaller scales}, MS-HVP matches or exceeds IG: on the hardest setting, MS-HVP at cost~10 achieves better accuracy than IG at cost~35 ($3.5\times$ cheaper, $p = 0.022$), and at matched cost MS-HVP wins by $1.2$~pp ($p < 0.001$). MS-HVP also outperforms IH by $1.5\%$--$13.9\%$ across all nine tasks at smaller scales. \emph{For circuit recovery}, these per-component gains translate to improved head recovery on IOI and Greater-Than benchmarks.

We release our code for reproducing the experimental results at \url{https://github.com/LZhang12345/attribution-patching-lies}. Our contributions are three-fold:
\setlength{\leftmargini}{1.8em}
\begin{itemize}
\setlength{\itemsep}{0.6em}
\setlength{\parsep}{0pt}
  \item \textbf{A reframe of attribution patching reliability.} We trace the dominant error to the downstream network response, not local nonlinearity, explaining why prior local fixes are structurally incomplete.

  \item \textbf{A scalable second-order correction.} HVP (and its multi-step variant MS-HVP) remains tractable at 8B+ parameters where existing refinement baselines become infeasible, the first second-order correction demonstrated at this scale.

  \item \textbf{A design space for attribution-patching refinement.} We organize prior methods by their goals, including estimation accuracy, feature interactions, and circuit faithfulness, clarifying which tool fits which use case.
\end{itemize}

\section{Related work}
\label{sec:related}

\paragraph{Patching methods for mechanistic localization.}
Activation patching~\citep{elhage2021mathematical, wang2023ioi} provides the causal reference intervention for mechanistic localization by measuring the effect of intervening on hidden activations.  Because this cost scales linearly with the number of candidate components, attribution patching~\citep{nanda2023ap} replaces many interventions with a first-order approximation and has become a standard localization primitive in automated circuit discovery~\citep{conmy2023acdc, syed2024eap}, sparse feature circuits~\citep{marks2025sparse}, and large-scale circuit tracing~\citep{anthropic2025circuits}.  Practical guides such as \citet{heimersheim2024howto} and methodological studies of activation-patching metrics and corruption choices~\citep{zhang2023bestpractices} note that patching often behaves approximately linearly while remaining sensitive to nonlinearities and design decisions, making attribution-patching reliability central to patching-based workflows. Recent large-scale circuit tracing~\citep{anthropic2025circuits,anthropic2026emotions} applies attribution to sparse autoencoder (SAE) features~\citep{marks2025sparse} rather than raw heads or neurons; HVP is directly applicable to any differentiable activation, though the SAE encoder's nonlinearity (ReLU or top-$k$) introduces an additional curvature source whose magnitude we leave to future work.

\paragraph{Attribution-patching reliability and circuit validation.}
Several failure modes of attribution patching are known: \citet{kramar2024atp} identify activation-region mismatch and cancellation; \citet{edin2025gim} show that softmax redistribution systematically biases gradient-based localization; \citet{meloux2025variance} document instability under prompt and hyperparameter variation; \citet{bereska2025open} identify gradient-attribution error as a standing open problem. Alternative attribution rules such as RelP~\citep{relp2025} and EAP-GP~\citep{zhang2025eapgp} improve faithfulness in circuit discovery. A complementary line of work evaluates recovered circuits against known causal structure: \citet{shi2024circuithypothesis} formalize faithfulness and minimality tests, \citet{mueller2025mib} and \citet{gupta2024interpbench} benchmark localization methods, and formal mechanistic-interpretability work~\citep{formal2026mechinterp, certified2026circuits} studies provable robustness. Generally, these efforts clarify failure modes and evaluation criteria but do not provide a general account of when attribution patching is numerically reliable or how to correct it across component types and architectures.
A complementary approach, Guided-AP~\citep{zhang2026correcting}, learns systematic attribution-patching errors from a small number of exact activation-patching measurements and uses them to correct unmeasured estimates within and across prompts.

\paragraph{Causal abstraction and higher-order analysis.}
\citet{geiger2025causalabstraction} place mechanistic interpretability in a broader causal-abstraction framework. Related work~\citep{geiger2024das,wu2023alpaca} studies whether models adhere to an interpretable causal structure, a question that remains orthogonal to the numerical reliability of the underlying attribution-patching estimates. Separately, higher-order attribution methods, such as Integrated Gradients~\citep{sundararajan2017ig}, Integrated Hessians~\citep{janizek2021hessians}, compositional curvature analyses~\citep{entesari2024curvature}, and influence-function HVPs~\citep{koh2017influence}, show that network-level curvature is both informative and tractable, but none provides a reliability account for patching at internal activations.

\section{Diagnosing and correcting attribution-patching error}
\label{sec:method}

In this section, we analyze when attribution patching is unreliable and how to correct it.
The key idea is that attribution patching error can be decomposed into two parts: a dominant quadratic term (computable from a single Hessian-vector product) and a smaller cubic remainder. We set up this decomposition (\S\ref{sec:error_decomp}), use it to define a reliability score that flags unreliable attributions (\S\ref{sec:bounds}), and then show why local activation curvature cannot predict the dominant error term before introducing network-level HVP, MS-HVP, and selective correction (\S\ref{sec:hvp}).

\subsection{Problem setup and error decomposition}
\label{sec:error_decomp}

Causal localization compares two matched inputs that differ in a controlled way (e.g., swapping a name or replacing a token) and asks which internal components account for the resulting change in output. For a given component (neuron, attention head, or residual-stream position), we compare its activation $a$ under one input to its perturbed counterpart $a' = a + \delta$ under the other. Activation patching measures the true causal effect of this substitution by intervening directly, whereas attribution patching approximates it using a single backward pass. Our focus is on characterizing the error introduced by this approximation.

Patching a single component changes a scalar output metric as follows. Let $M: \mathbb{R}^d \to \mathbb{R}$ denote the scalar score read out from the model output as a function of the activation being patched (e.g., a logit difference or the log-probability of a target token), and let $a \in \mathbb{R}^d$ denote the activation at the component of interest. Throughout this local analysis, we treat the rest of the input context and
model computation as fixed. \emph{Activation patching} measures the true effect of replacing $a$ with the counterfactual $a' = a + \delta$: $\Delta = M(a + \delta) - M(a)$. \emph{Attribution patching} approximates this via a first-order Taylor expansion:
\begin{equation}
  \hat{\Delta} = \nabla_a M \cdot \delta\,.
  \label{eq:ap}
\end{equation}
Here $\nabla_a M$ is the gradient of the scalar metric with respect to the activation $a$, so $\hat{\Delta}$ is the first-order attribution-patching estimate along the patch direction $\delta$. The approximation error is $E = \Delta - \hat{\Delta}$.  By Taylor's theorem with integral remainder,
\begin{equation}
  E = \underbrace{\tfrac{1}{2}\,\delta^{\!\top} H \delta}_{\text{dominant (Hessian) term}} + \underbrace{\Phi(\delta)}_{\text{remainder}},
  \qquad
  H = \nabla_a^2 M(a).
  \label{eq:error_decomp}
\end{equation}
The quantity $\delta^{\!\top} H \delta$ is the Hessian quadratic form along the patching direction $\delta$, i.e., the second-order curvature of the scalar metric in the direction induced by the patch.  

To bound the higher-order terms in Eq.~\eqref{eq:error_decomp}, we adopt a path-local Lipschitz-Hessian condition, the standard route to a cubic Taylor remainder in second-order optimization~\citep{nesterov2006cubic, cartis2011arc} and a common tool in higher-order analyses of neural networks and self-attention~\citep{janizek2021hessians, entesari2024curvature,jukic2025robustness, kim2021lipschitz, castin2024smooth}.  We use it only along the patch segment, not as a deployment-computable global constant.

\begin{assumption}[Local third-order smoothness]
\label{ass:local_smooth}
The scalar metric $M$ is twice continuously differentiable on a neighborhood of the line segment $\{a + t\delta : t \in [0,1]\}$, and there exists a finite constant $L_3$ such that
\begin{equation}
  \bigl\|\nabla_a^2 M(a + t\delta) - \nabla_a^2 M(a)\bigr\|_{\mathrm{op}}
  \leq L_3\, t \|\delta\|
  \qquad \text{for all } t \in [0,1].
  \label{eq:local_smooth}
\end{equation}
\end{assumption}

\paragraph{How the assumption applies to patching.}
The assumption is local: it applies only along the patch segment and yields the remainder bound $|\Phi(\delta)| \leq L_3 \|\delta\|^3 / 6$.  Standard transformer components satisfy it in the relevant regime: exact GeLU and SiLU have bounded third derivatives, softmax is smooth, and normalization layers are smooth away from zero-variance inputs. Empirically, we estimate $L_3$ along the patch path by probing the Hessian at three interpolation points (Appendix~\ref{app:l3-empirical}): the resulting cubic bound holds for $82\%$ of GPT-2 and $95\%$ of Gemma-2-2B component--prompt pairs, with median slack factors of $3.4\times$ and $7.6\times$ respectively.

\subsection{Reliability score and error bounds}
\label{sec:bounds}

We now turn the decomposition into a diagnostic.  The goal is to decide, before running activation patching, whether a first-order attribution-patching score is likely to be trustworthy for a given component. The quadratic term is computable
from one HVP, but its raw magnitude is hard to interpret without a scale; we therefore normalize it by the first-order estimate. The result is a relative reliability score that tracks attribution patching's relative error up to the cubic remainder and is used later to flag components for selective HVP/MS-HVP correction.

\begin{definition}[Reliability score]
\label{def:rscore}
For a component with Hessian $H = \nabla_a^2 M(a)$, perturbation $\delta$, and nonzero first-order estimate $\hat\Delta \neq 0$, the \emph{reliability score} is $\tilde{R} = |\delta^{\!\top} H \delta| / (2\,|\hat{\Delta}|)$.
\end{definition}

The denominator is chosen to match the practical question: whether the attribution patching estimate itself is stable.  A small $\tilde R$ means the leading omitted term is small relative to the first-order estimate, while $\tilde R \approx 1$ means downstream curvature can change the attribution's magnitude or sign.  Normalizing by the true effect $\Delta$ would require activation patching, so it is not available as a screening diagnostic; in near-zero cases, the absolute error bounds from Eq.~\eqref{eq:error_decomp} are the right object to inspect.

\begin{proposition}[Local attribution-patching error bound via the reliability score]
\label{thm:main}
Assume Assumption~\ref{ass:local_smooth}, $\hat{\Delta} \neq 0$, and $\delta^{\!\top}H\delta \neq 0$. Let $\alpha = \tfrac{L_3 \|\delta\|^3}{3\,|\delta^{\!\top} H \delta|}$ be the higher-order slack parameter.  Then $\bigl||E|/|\hat{\Delta}| - \tilde R\bigr| \leq \alpha \tilde R$. If additionally $\alpha < 1$, then $(1 - \alpha)\tilde{R} \leq |E|/|\hat{\Delta}|
\leq (1 + \alpha)\tilde{R}$.
\end{proposition}

The full proof, including verification of the remainder bound and a discussion of path-local smoothness for softmax and normalization layers, is in Appendix~\ref{app:proofs}.

In practice, $\tilde R \ll 1$ means attribution patching is accurate; $\tilde R \approx 1$ means the error may rival the estimate itself. The slack $\alpha$ packages the uncomputed cubic remainder: small $\alpha$ corresponds to perturbations well within the Taylor convergence radius, making $\tilde R$ a tight proxy for relative error; otherwise $\tilde R$ remains a useful diagnostic. Degenerate cases are handled naturally: if $\hat\Delta = 0$, use absolute bounds; if $\delta^{\!\top}H\delta = 0$, the quadratic term vanishes and $|E| \leq \tfrac{L_3}{6}\|\delta\|^3$.

\subsection{Network-level HVP correction}
\label{sec:hvp}
\label{sec:nonlocality}

The quadratic error term $\tfrac{1}{2}\delta^{\!\top}\!H\delta$ is a network-level quantity: it depends on how the patched activation propagates through the full computational graph, not only on local activation curvature.  Correcting it therefore requires network-level information.  We first show that local curvature corrections fail, then introduce HVP and MS-HVP corrections that compute the relevant term directly.

\begin{wrapfigure}{R}{0.5\textwidth}
\centering
\vspace{-12pt}
\includegraphics[width=0.48\textwidth]{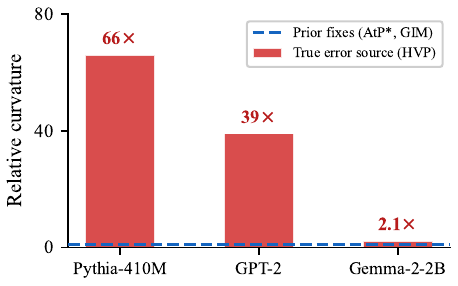}
\vspace{-6pt}
\caption{\textbf{Network--local curvature gap.}
Full-network curvature vs.\ local component curvature across three models; prior fixes (AtP$^*$~\citep{kramar2024atp}, GIM~\citep{edin2025gim}) use only the local quantity.}
\label{fig:curvature_gap}
\vspace{-10pt}
\end{wrapfigure}

\paragraph{Why local corrections fail.}
A natural first attempt is to use the \emph{local} activation curvature.  For a pre-activation MLP neuron with activation function $f$, one might estimate $E \approx \tfrac{1}{2}f''(z)\delta^2$.  This should fail for a structural reason: a neuron's contribution must traverse the rest of the network, and most of that downstream path is linear (residual additions, LayerNorm in its operating regime, the unembedding).  Linear operations do not compound curvature, so $|d^2 M / da^2|$ inherits little of $|f''(z)|$ unless a downstream nonlinearity happens to align with the patch direction. Cross-layer interactions further introduce curvature invisible to local analysis.  We therefore expect a large gap between $|d^2\!M/dz^2|$ and $|f''(z)|$, and confirm this empirically: the network-level Hessian exceeds local curvature by $2$--$66\times$ across three model families, with near-zero correlation ($r = 0.04$); using $f''(z)$ for correction produces $>$750\% error -- worse than no correction at all (Figure~\ref{fig:curvature_gap}).

\paragraph{Network-level correction.}
This gap motivates computing the network-level second-order correction directly.

\begin{definition}[HVP-corrected attribution patching]
\label{def:hvp}
The HVP-corrected estimate is
\begin{equation}
  \hat{\Delta}_{\text{hvp}}
  = \nabla_a M \cdot \delta + \tfrac{1}{2}\,\delta^{\!\top} H \delta = \hat{\Delta} + \tfrac{1}{2}\,\delta^{\!\top} H \delta\,.
  \label{eq:hvp}
\end{equation}
\end{definition}

Computing $\delta^{\!\top}\!H\delta$ requires the Hessian-vector product $H\delta$, which we obtain via two backward passes~\citep{pearlmutter1994hvp}: (1)~compute $g = \nabla_a M$ while retaining the gradient computation graph so that $g$ itself can be differentiated; (2)~differentiate $g \cdot \delta$ with respect to $a$ to obtain $H\delta$. The dot product $\delta^{\!\top}(H\delta)$ then gives the quadratic form.  This costs $\sim 3\times$ a single backward pass (one forward + two backward).\footnote{Measured wall-clock overhead on Pythia-410M is $2.8\times$ rather than $3\times$ due to caching effects from the retained computation graph.}

\begin{corollary}[Residual error after HVP correction]
\label{thm:hvp_error_main}
The residual error after HVP correction satisfies
$|E_{\text{hvp}}| = |\Delta - \hat\Delta_{\text{hvp}}| = |\Phi(\delta)| \leq L_3 \|\delta\|^3 / 6$.
The proof is in Appendix~\ref{app:hvp_error}.
\end{corollary}

Corollary~\ref{thm:hvp_error_main} shows that HVP reduces the error of attribution patching from $O(\|\delta\|^2)$ to $O(\|\delta\|^3)$, one order tighter as $\delta \to 0$. However, when $L_3 \|\delta\|$ is comparable to the Hessian scale, the cubic remainder dominates and a single Taylor step overshoots.  To address this, we introduce \emph{Multi-Step HVP} (MS-HVP), which splits the patch into $K$ equal sub-steps and applies a second-order correction at each intermediate point along the path from $a$ to $a{+}\delta$.  Setting $K{=}1$ recovers standard HVP; increasing $K$ shrinks each sub-step's remainder at the cost of additional HVP evaluations. Algorithm~\ref{alg:mshvp} summarizes the procedure.

\begin{algorithm}[H]
\caption{Multi-step HVP attribution patching}
\label{alg:mshvp}
\begin{algorithmic}[1]
\Require Model $M$, activation $a$, perturbation $\delta$, sub-steps $K$
\State $\hat\Delta_{\text{ms}} \leftarrow 0$,\; $s \leftarrow \delta/K$
\For{$k = 1, \ldots, K$}
  \State $a_{k-1} \leftarrow a + (k-1)s$
  \State $g_k \leftarrow \nabla_a M(a_{k-1})$ \Comment{Forward + first backward}
  \State $v_k \leftarrow \nabla_a(g_k \cdot s)$ \Comment{HVP via second backward}
  \State $\hat\Delta_{\text{ms}} \leftarrow \hat\Delta_{\text{ms}} + g_k \cdot s + \tfrac{1}{2}\,s \cdot v_k$
\EndFor
\State \textbf{return} $\hat\Delta_{\text{ms}}$
\end{algorithmic}
\end{algorithm}

Concretely, MS-HVP evaluates the gradient and Hessian at each intermediate activation $a_{k-1} = a + \tfrac{k-1}{K}\delta$ and accumulates:
\begin{equation}
  \hat\Delta_{\text{ms}}
  = \sum_{k=1}^{K} \Bigl[
      \nabla_a M(a_{k-1}) \cdot \tfrac{\delta}{K}
    + \tfrac{1}{2K^2}\,\delta^{\!\top} H(a_{k-1}) \delta
    \Bigr].
  \label{eq:mshvp}
\end{equation}
Because each sub-step's cubic remainder scales as $\|\delta/K\|^3$ and there are $K$ such steps, the aggregate residual satisfies $|E_{\text{ms}}| \leq L_3 \|\delta\|^3 / (6 K^2)$ by standard quadrature analysis.  MS-HVP thus trades linearly more HVP evaluations for a quadratic reduction in the Taylor remainder, predicting diminishing returns as $K$ grows.  We test this cost--accuracy prediction empirically in \S\ref{sec:experiments}.

\paragraph{Selective HVP: deciding which components to correct.}
MS-HVP improves the accuracy of individual corrections; a separate question is \emph{which} components need correction at all. In practice, one rarely applies HVP to every candidate component.  The following result formalizes the aggregate error of the Screen $\rightarrow$ Flag $\rightarrow$ Fix workflow: compute $\tilde R$ for all components, apply HVP correction only to those with $\tilde R \geq \tau$, and trust attribution patching for the rest.

\begin{proposition}[Selective-HVP pipeline guarantee]
\label{thm:selective_pipeline_main}
Consider components $i=1,\dots,n$ patched independently on the same input.  Let $\Delta_i$ denote the true effect of patching only component $i$, $\hat\Delta_i$ its attribution-patching estimate, and $c_i := \tfrac{1}{2}\delta_i^{\!\top} H_{ii} \delta_i$ its diagonal second-order correction.  For a threshold $\tau > 0$, define $S_{\mathrm{ok}} := \{i : \tilde R_i < \tau\}$, $S_{\mathrm{flag}} := \{i : \tilde R_i \ge \tau\}$, and the selective estimate
\begin{equation}
  \hat\Delta_{\mathrm{sel}}
  :=
  \sum_{i=1}^n \hat\Delta_i
  +
  \sum_{i \in S_{\mathrm{flag}}} c_i.
  \label{eq:selective_est}
\end{equation}
with the convention that $\tilde R_i = +\infty$ when $\hat\Delta_i = 0$. If Assumption~\ref{ass:local_smooth} holds along each componentwise patching segment with constants $L_{3,i}$, then
\begin{equation}
  \bigl|\textstyle\sum_{i=1}^n \Delta_i - \hat\Delta_{\mathrm{sel}}\bigr|
  \;\le\;
  \tau \sum_{i \in S_{\mathrm{ok}}} |\hat\Delta_i|
  \;+\;
  \tfrac{1}{6}\sum_{i=1}^n L_{3,i}\|\delta_i\|^3.
  \label{eq:selective_pipeline_bound}
\end{equation}
The proof is in Appendix~\ref{app:proofs}.
\end{proposition}

The independent-patching assumption matches the standard practice in circuit discovery: tools such as ACDC~\citep{conmy2023acdc} and EAP~\citep{syed2024eap} consume per-component rankings. Joint patches add off-diagonal Hessian contributions $\delta_i^{\!\top} H_{ij} \delta_j$; these affect aggregate circuit-level metrics but not per-component ranking (measuring them would require $O(n^2)$ HVPs). Appendix~\ref{app:selective_pipeline} gives an exact identity comparing selective and full diagonal HVP plus a sign-aware refinement.

\section{Experiments}
\label{sec:experiments}

We evaluate whether HVP corrects attribution-patching error broadly, how it compares to existing second-order methods at matched cost, and whether per-component accuracy gains translate to improved circuit recovery.

\subsection{Setup}
\label{sec:setup}

All experiments use TransformerLens~\citep{nanda2022transformerlens} with 95\% paired bootstraps (1{,}000 resamples). Full details are in Appendices~\ref{app:setup} and ~\ref{app:metric_identity}. We evaluate five model families (124M--9B parameters): GPT-2~\citep{radford2019gpt2}, Pythia (410M, 2.8B)~\citep{biderman2023pythia}, Qwen2.5-1.5B~\citep{qwen2.5}, Gemma (2B, 9B)~\citep{gemma2024}, and Llama-3.1-8B~\citep{llama3_2024}. We report attention heads in all models, pre-activation MLP neurons in GeLU models, and MLP layer outputs in Qwen. For factual and Greater-Than tasks, perturbations replace position~3 with a random vocabulary token; for IOI tasks, the corruption is the standard name-swap~\citep{wang2023ioi}, a naturalistic perturbation. In both cases, $\delta$ is the resulting activation shift and $M$ is the log-probability of the correct next token. Our primary metric is \emph{top-5 relative error}: the mean absolute error of a method's scores on the five components with largest activation-patching effect, normalized by the ground-truth range. Confidence intervals are reported in the appendix tables.

Our IG baseline~\citep{sundararajan2017ig} integrates the gradient along the patch path using $S$ interpolation steps (cost~$S$ per head); we use \emph{per-head} IG, which differs from all-at-once EAP-IG~\citep{syed2024eap} (Appendix~\ref{app:ig_distinction}). Cost is measured in backward passes per head per prompt: attribution patching $= 1$, Std HVP $= 2$, MS-HVP $K = 2K$ (backward passes only; wall-clock overhead is moderately higher; see Appendix~\ref{app:wallclock}).

Our model-task selection follows the theoretical prediction of \S\ref{sec:method}: attribution-patching error should be largest where the network-level quadratic form $\delta^{\!\top}\!H\delta$ dominates.  Pythia-410M IOI is the predicted high-error regime (Taylor radius violated); GPT-2 IOI is the predicted low-error regime (attribution patching already accurate); Gemma-2-2B factual sits between these extremes. Seven additional model-task pairs and four tasks with larger-scale models provide cross-architecture coverage (Table~\ref{tab:master}).

\begin{table}[t]
\centering
\caption{\textbf{Comparison of error rate.}  Top-$5$ relative error (\%) across fourteen model-task pairs and six methods (lower is better; Std HVP $=$ MS-HVP $K{=}1$).
Column $N$: number of evaluation prompts;
$K$: HVP sub-steps; $S$: IG interpolation steps.}
\label{tab:master}
\small
\begin{tabular*}{\textwidth}{@{\extracolsep{\fill}}l c c c c c c c}
\toprule
Model / Task & $N$ & Attrib.\ Pat.\ & Std HVP & MS-HVP ($K{=}5$) & IG ($S{=}10$) & IH-PI & GIM \\
\midrule
\multicolumn{8}{@{}l}{\emph{Foreground (full method coverage)}}\\
Pythia-410M IOI$^{*}$    & 280 & 22.34 & 47.48$^{\dagger}$ & 18.03$^\star$ & \textbf{17.26}$^\star$ & 22.22 & 34.07$^\circ$ \\
GPT-2 IOI                & 100 & 18.07 & 12.00$^\star$     & \textbf{2.97}$^\star$  & 3.30$^\star$  & 10.15$^\star$ & 54.71$^\circ$ \\
Gemma-2-2B factual       & 55  & 51.71 & 43.16$^\star$     & 40.66$^\star$ & \textbf{40.57}$^\star$ & 44.41$^\star$ & 77.68$^\circ$ \\
\midrule
\multicolumn{8}{@{}l}{\emph{Cross-model coverage (Std HVP / IG / GIM)}}\\
Pythia-410M factual      & 55  & 68.65 & 65.58$^\star$    & 64.41$^\star$  & \textbf{63.91}$^\star$ & 65.94$^\star$   & 75.07$^\circ$ \\
Pythia-410M GT           & 200 & 35.68 & 25.55$^\star$    & \textbf{21.47}$^\star$  & 24.62$^\star$ & 28.48$^\star$   & 39.65$^\circ$ \\
GPT-2 GT                 & 200 & 44.06 & 33.80$^\star$    & \textbf{23.33}$^\star$  & 32.98$^\star$ & 37.25$^\star$   & 41.35$^\star$ \\
Gemma-2-2B IOI           & 50  & 16.53 & 9.07$^\star$     & \textbf{6.58}$^\star$   & 6.63$^\star$  & 10.54$^\star$   & 25.65$^\circ$ \\
Qwen2.5-1.5B IOI         & 50  & 14.80 & 5.69$^\star$     & \textbf{2.94}$^\star$  & 3.16$^\star$  & 8.13$^\star$   & 24.75$^\circ$ \\
Qwen2.5-1.5B factual     & 35  & 41.57 & 35.79$^\star$    & 32.06$^\star$  & \textbf{31.95}$^\star$ & 35.52$^\star$   & 63.25$^\circ$ \\
Gemma-2-2B GT            & 50  & 7.41  & \textbf{0.62}$^\star$ & ---           & ---           & ---          & --- \\
\midrule
\multicolumn{8}{@{}l}{\emph{Larger scale (IG/IH infeasible)}}\\
Llama-3.1-8B IOI         & 50  & 19.78 & 8.42$^\star$ & \textbf{3.54}$^\star$  & $\ddagger$ & $\ddagger$ & 30.56$^\circ$ \\
Llama-3.1-8B factual     & 55  & 53.19 & 50.50$^\star$    & \textbf{48.69}$^\star$  & $\ddagger$ & $\ddagger$ & 71.35$^\circ$ \\
Llama-3.1-8B GT          & 200 & 19.51 & 7.75$^\star$   & \textbf{6.51}$^\star$  & $\ddagger$ & $\ddagger$ & 21.17$^\circ$ \\
Gemma-2-9B factual       & 55  & 60.82 & 56.97$^\star$    & \textbf{56.93}$^\star$  & $\ddagger$ & $\ddagger$ & 89.50$^\circ$ \\
\bottomrule
\end{tabular*}

{\raggedright \footnotesize
$^\star$Significantly better than AP ($p < 0.05$, paired bootstrap).\quad
$^\circ$Significantly worse.\quad
$^{*}$Pathological; recovered by MS-HVP $K\geq 3$.\quad
$^{\dagger}$Catastrophic; see \S\ref{sec:broad_gains}.\quad
$^{\ddagger}$Infeasible (${\sim}25$ GPU-days/task).\quad
---\,Tokenizer-incompatible: the greater\_than task requires patching a single token representing a two-digit number (e.g.\ ``42''), but Gemma's tokenizer splits it into multiple tokens.\\
\emph{Note:} AtP$^{*}$~\citep{kramar2024atp} and RelP~\citep{relp2025} are
omitted due to different granularity. \par}
\end{table}

\subsection{Broad gains from HVP correction}
\label{sec:broad_gains}

Consistent with the curvature-gap analysis (\S\ref{sec:nonlocality}, Figure~\ref{fig:curvature_gap}), local activation curvature is a poor predictor of attribution-patching error. In contrast, the reliability score $\tilde{R}$ tracks network-level curvature and accurately localizes failure regions. Figure~\ref{fig:diagnostic}
shows the reliability-score diagnostic on Pythia-410M, while Figure~\ref{fig:nonlocality} shows that $\tilde{R}$ sharply concentrates in the IOI-circuit layers of Pythia-410M, precisely where attribution patching incurs its largest errors. Applying the network-level HVP correction substantially reduces error throughout these layers. Full HVP correction results are deferred to Appendix~\ref{app:full-per-model-results}.

\begin{figure}[ht]
\begin{minipage}[t]{0.6\linewidth}
    \includegraphics[width=\linewidth]{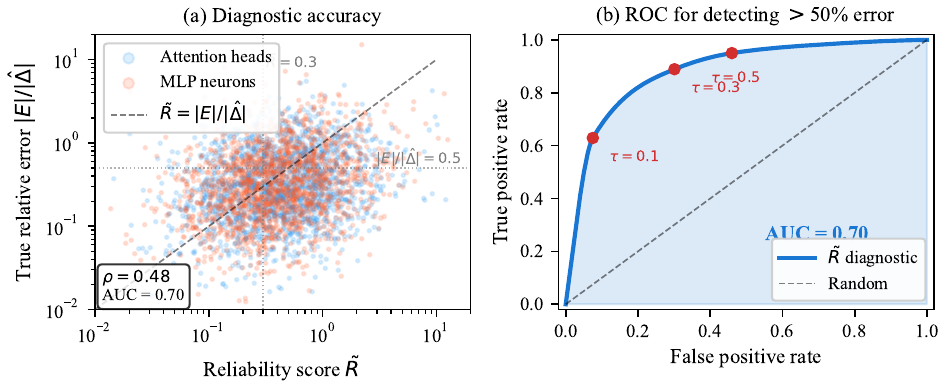}
    \caption{Reliability score $\tilde R$ as a diagnostic for attribution-patching failure on
    Pythia-410M.  \textbf{(a)}~Scatter of $\tilde R$ vs.\ true relative error for
    attention heads (blue) and MLP neurons (red), $n = 3{,}600$.  Spearman
    $\rho = 0.48$.  Dashed lines mark the recommended thresholds.
    \textbf{(b)}~ROC curve for detecting $>$50\% relative error.
    AUC = $0.70$ $[0.67, 0.73]$.  At $\tilde R > 0.3$ (marked), recall is 89\%
    and precision 36\%, flagging 40\% of components.}
    \label{fig:diagnostic}
\end{minipage}
\hfill
\begin{minipage}[t]{0.36\linewidth}
\includegraphics[width=\linewidth]{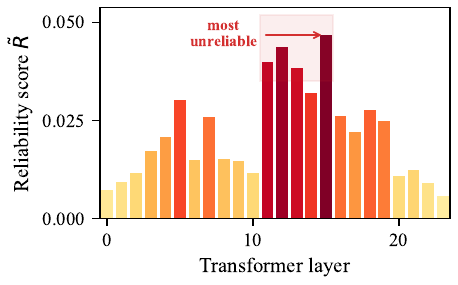}
\caption{\textbf{Reliability score by layer.}
$\tilde{R}$ by transformer layer (Pythia-410M IOI, heads with top-quartile causal effect). Error concentrates in the IOI-circuit layers (11--15).}
\label{fig:nonlocality}
\end{minipage}
\end{figure}

These improvements translate into broad empirical gains across architectures and tasks (Table~\ref{tab:master}). Standard HVP correction ($K{=}1$, cost~2) reduces top-5 relative error by $72$--$90\%$ on attention heads and $56$--$67\%$ on pre-activation MLP neurons. In contrast, gains on post-activation neurons are small ($2$--$7\%$), consistent with the near-linearity of activations after the nonlinearity has already been applied. The same qualitative pattern holds across all evaluated architectures, including SwiGLU and GeGLU models such as Qwen and Gemma. The strongest results occur in low-to-moderate curvature regimes. For example, on Gemma-2-2B Greater-Than, a single HVP pass reduces the top-5 relative error from $7.41\%$ to $0.62\%$, approaching exact recovery of the activation-patching ranking.

Beyond aggregate metrics, HVP correction also improves circuit recovery. On GPT-2 Greater-Than, MS-HVP ($K{=}5$) increases top-5 head overlap with activation-patching ground truth from $70.1\%$ to $83.2\%$ (+13.1~pp), outperforming both IG and GIM. Similar improvements hold on Pythia-410M Greater-Than, where all second-order methods converge to comparable rankings while consistently outperforming attribution patching.

\paragraph{Larger-scale models.}
At 8B+ parameters, per-head IG at $S{=}10$ requires ${\sim}25$ GPU-days per task ($S \times N_{\text{heads}}$ backward passes per prompt; Appendix~\ref{app:ig_distinction}), and IH has not been demonstrated at this scale.
HVP at $K{=}1$ (cost~2) running in $2.4$--$28.7$~GPU-hours is the only second-order correction that fits in a typical research budget.
Std HVP reduces error by $5$--$60\%$ across the four tasks with larger-scale models; on Llama-3.1-8B IOI, MS-HVP $K{=}5$ further pushes error from $8.42\%$ to $3.54\%$, an $82\%$ reduction from attribution patching, confirming that multi-step correction remains beneficial at 8B scale.

\paragraph{Pathological high-curvature regime.}
Pythia-410M IOI is the unique scenario where $|\delta|$ exceeds the Taylor convergence radius: Std HVP overshoots to $47.48\%$, validating Cor.~\ref{thm:hvp_error_main}.
Multi-step composition restores stability at the predicted $\mathcal{O}(1/K^2)$ rate (Figure~\ref{fig:correction}a), with knee at $K{=}5$ and diminishing returns beyond $K{=}10$.
Direct higher-order alternatives (third-order Taylor, finite-difference HVP) fail in this regime but not in low-$|\delta|$ tasks (Appendix~\ref{app:full-per-model-results}).

\subsection{Matched-compute comparisons}
\label{sec:comparisons}

We compare HVP against existing second-order methods at matched compute budgets on the nine tasks with smaller-scale base models where all baselines are feasible.

\emph{Vs.\ IG.}  At cost~10, paired-bootstrap comparisons yield seven ties ($p > 0.1$) and two MS-HVP wins (GPT-2
IOI, $p < 0.001$; Qwen IOI, $p = 0.041$).  On the pathological case, MS-HVP
$K{=}5$ at cost~10 beats IG $S{=}35$ at cost~35 ($\Delta = -0.44$~pp, $p = 0.022$), demonstrating better accuracy at $3.5\times$ less compute.  At matched
cost~${\sim}40$, a per-head median achieves $17.57\%$ vs.\ IG's $18.77\%$
($p < 0.001$; Appendix~\ref{app:hvp_vs_ig}).  In short, MS-HVP matches or
exceeds IG while extending to scales where IG cannot run.

\emph{Vs.\ IH.}  The gap is structural: IH's
path-integrated weighting applies $\int_0^1 2t(1{-}t)\,dt = 1/3$ to the Hessian term, yielding an effective $1/4$ coefficient on
$\delta^{\!\top}\!H\delta$, which is half the $1/2$ needed for exact Taylor
correction.
This under-correction is by design: IH targets feature-interaction attribution, not error correction.
The resulting gap is consistent across all nine tasks ($1.5$--$13.9\%$; Appendix~\ref{app:hvp_vs_ih}).

\emph{Vs.\ GIM.}  GIM~\citep{edin2025gim} targets edge-level faithfulness metrics~\citep{mueller2025mib} rather than per-component accuracy.
On our per-head metrics, GIM is significantly worse than attribution patching on 12/13 tasks ($p<0.001$); the sole exception is GPT-2 Greater-Than, where GIM improves top-5 error by $2.7\%$ but still degrades head recovery (Appendix~\ref{app:hvp_vs_gim}).
This does not invalidate GIM for its intended use case but highlights that edge-level and component-level accuracy are distinct objectives -- a key finding of our design-space analysis (Table~\ref{tab:recommendations} in the appendix organizes methods by their target objective).

\begin{figure}[t]
\centering
\includegraphics[width=\linewidth]{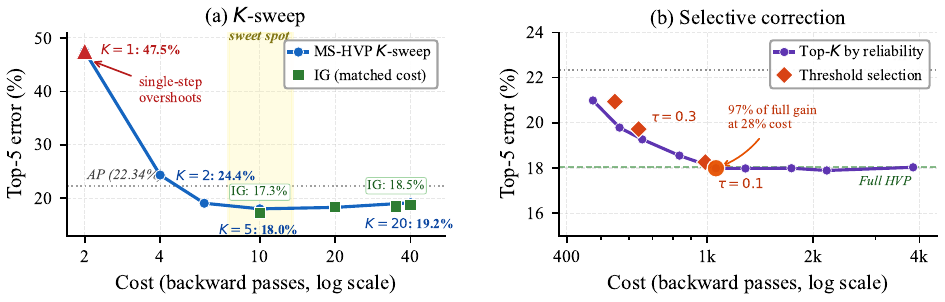}
\vspace{-8pt}
\caption{\textbf{Multi-step correction and selective workflow.}
\textbf{(a)}~$K$-sweep on Pythia-410M IOI: top-5 relative error vs.\ number
of sub-steps $K$ for MS-HVP (blue) and IG (green squares) at matched
per-step cost.
\textbf{(b)}~Selective-HVP on GPT-2 IOI: top-5 error vs.\ total backward-pass cost as more components are corrected, ranked by $\tilde R$.
$\tilde R$-based selection (orange) vs.\ random baseline (gray).}
\label{fig:correction}
\end{figure}

\subsection{Selective workflow and diagnostics}
\label{sec:discussion}
The reliability score $\tilde R$ provides a practical diagnostic for deciding when correction is worthwhile, completing the \emph{Screen-Flag-Fix} pipeline.
\paragraph{Selective correction.}  Applying HVP only to $\tilde R$-flagged components (Figure~\ref{fig:correction}b) captures $97\%$ of full-HVP gain at $28\%$ of the cost ($\tau{=}0.1$), validating Proposition~\ref{thm:selective_pipeline_main}. The $\tilde R$-ranked selection substantially outperforms random selection (gray baseline in Figure~\ref{fig:correction}b), confirming that $\tilde R$ identifies the components that benefit most from correction. The full operating curve across three models is in Appendix~\ref{app:selective_tradeoff}.

\paragraph{Diagnostic accuracy.}
Figure~\ref{fig:rtilde_strat} evaluates the reliability score $\tilde{R}$ after stratifying components by perturbation magnitude $|\delta|$ and true effect magnitude $|f_{\mathrm{true}}|$. Q1--Q4 denote quartiles of the stratified variables. Although overall AUROC varies across models, performance is consistently strong in the regimes that matter most: for Q2--Q4, $\tilde{R}$ achieves AUROC above 0.97 across all evaluated models under both stratifications. Thus, the reliability score accurately identifies components where attribution patching is likely to be unreliable, even when applied without model-specific tuning.

\begin{figure}[!t]
    \centering
    \includegraphics[width=\linewidth]{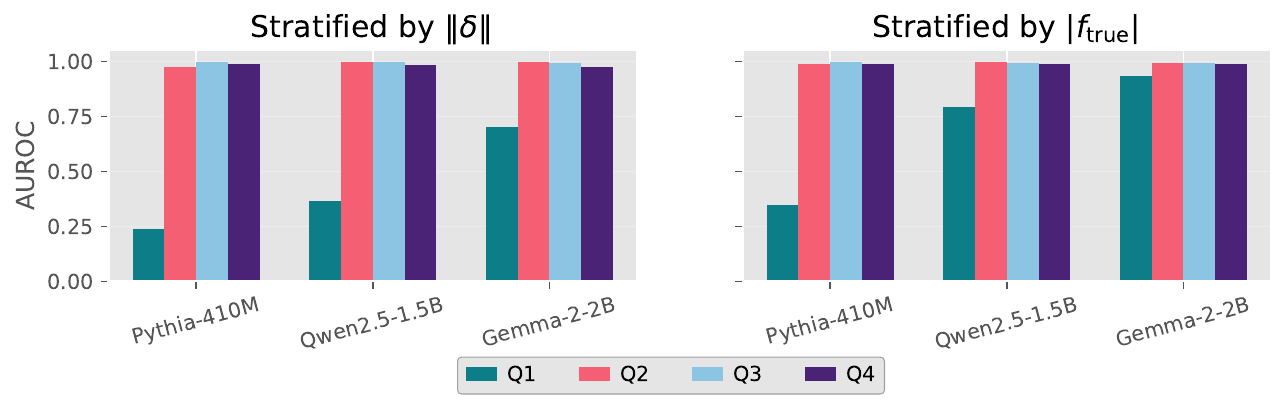}
    \caption{AUROC of $\tilde{R}$ stratified by $\|\delta\|$ quartile and
    $|f_{\mathrm{true}}|$ quartile.  Q1 denotes the smallest-norm quartile.}
    \label{fig:rtilde_strat}
\end{figure}

\paragraph{Robustness analysis.}
Supplementary experiments confirm that HVP's gains are robust across corruption types (Appendix~\ref{app:corruption_robustness}), semantic perturbations (Appendix~\ref{app:semantic_corruption}), and evaluation sample sizes (Appendix~\ref{app:stability}). Across random-token, cross-prompt resample, and zero corruptions, HVP consistently reduces attribution-patching error by roughly 78--90\% on GPT-2, Gemma-2-2B, and Pythia-1B. Gains are weaker on the pathological Pythia-410M IOI setting (36–50\%) and on one Pythia-1B random-token run with unusually large perturbation norms (12\%).. HVP also generalizes beyond synthetic token replacements: under semantically coherent entity-swap corruptions, it reduces error by 54\% on GPT-2 and 59\% on Pythia-410M, indicating that the second-order correction remains effective for larger, structured activation perturbations. Finally, aggregate error-reduction estimates stabilize after approximately 10--15 evaluation prompts and converge to values consistent with the main results, suggesting that the reported gains are not driven by small-sample effects.

\subsection{Circuit-recovery payoff}
\label{sec:act4}

Better component estimates translate to improved circuit discovery. We evaluate this by ranking all attention heads by attribution magnitude and thresholding at the known circuit size to measure overlap with the activation-patching ground truth across multiple architectures and tasks, including GPT-2 IOI/Greater-Than~\citep{wang2023ioi, hanna2023greater}, Pythia-410M Greater-Than and Gemma-2-2B Greater-Than. Attribution patching already provides a strong global ranking: Kendall $\tau$ rank correlations range from $0.35$ on Pythia-410M to $0.72$ on Gemma-2-2B factual. HVP matches or slightly improves these global correlations. However, HVP's primary gains concentrate precisely at the ranking boundaries where circuit-membership decisions are made. The full ranking performance is deferred to Appendix~\ref{app:ranking}.

On GPT-2 IOI, HVP recovers all 20 ground-truth heads at the canonical
boundary, compared to\ $95\%$ for attribution patching.
We present a concrete example in Appendix~\ref{app:case_study} to illustrate why: head L4H11, a known duplicate-token head, ranks 7th by ground truth but 27th by attribution patching ($4.4\times$ underestimate of its causal effect). HVP partially recovers the true score and promotes L4H11 to rank~12; the reliability score successfully flags it with $\tilde R = 1.32 \gg 0.3$, correctly identifying it as a correction target. Appendix~\ref{app:head_recovery} compares the pattern of boundary improvement on Greater-Than task across model architectures.

These gains also improve MIB's circuit-faithfulness metrics~\citep{mueller2025mib} and SAE features. We defer more experimental results to Appendix~\ref{app:mib} and Appendix~\ref{app:sae_feature}, respectively.
\section{Conclusion}
\label{sec:conclusion}

We showed that attribution-patching error is dominated by the downstream network response, not local nonlinearity at the patched component.
This reframing explains why prior local corrections are structurally incomplete and motivates a simple fix: a Hessian-vector product that captures the missing curvature.
The resulting Screen--Flag--Fix pipeline lets practitioners keep the speed of attribution patching where it is already accurate and apply a targeted second-order correction where it is not, and scale to models where existing refinement methods are infeasible.
Natural next steps include extending the framework to multi-token semantic corruptions, integrating with circuit-completeness and minimality verification, and scaling the correction to sparse-autoencoder features.

The current framework focuses on single-token perturbations. However, extending to multi-token semantic corruptions is a natural next direction. While we demonstrate consistent gains up to 9B parameters, verifying behaviour at yet larger scales remains an open opportunity. More broadly, HVP improves component-level attribution accuracy but does not by itself address circuit completeness or minimality. Integrating with verification tools is an exciting avenue for future work.

\begin{ack}
The authors are grateful to the anonymous reviewers for their thoughtful and constructive feedback. LZ would like to thank Tata Consultancy Services for supporting him through the TCS Presidential Fellowship at Carnegie Mellon University. JW conducted this research independently and received no external funding or financial support for this work.
\end{ack}

\bibliographystyle{unsrtnat}
\bibliography{references}


\appendix
\newpage
\appendix

\section{Notation}
\label{app:notation_summary}
Table~\ref{tab:notation} summarizes the core notation used in the main text.

\begin{table}[ht]
\caption{Core notation used in the main text.}
\label{tab:notation}
\centering
\small
\begin{tabular}{p{0.23\linewidth}p{0.67\linewidth}}
\toprule
Symbol & Meaning \\
\midrule
$M(a)$ & Scalar output metric as a function of an internal activation \\
$a$ & Clean activation at the component being patched \\
$a' = a + \delta$ & Counterfactual activation after patching \\
$\delta$ & Patching perturbation \\
$\Delta$ & True activation-patching effect, $M(a+\delta)-M(a)$ \\
$\hat{\Delta}$ & First-order attribution-patching estimate, $\nabla_a M \cdot \delta$ \\
$E$ & Attribution-patching approximation error, $\Delta - \hat{\Delta}$ \\
$H$ & Network-level Hessian, $\nabla_a^2 M(a)$ \\
$\Phi(\delta)$ & Third-order remainder term in the Taylor expansion \\
$L_3$ & Hessian Lipschitz constant on the patching segment \\
$\tilde{R}$ & Reliability score, $|\delta^\top H \delta|/(2|\hat{\Delta}|)$ \\
$\alpha$ & Validity parameter controlling remainder size relative to the quadratic term \\
$\hat{\Delta}_{\text{hvp}}$ & HVP-corrected estimate, $\hat{\Delta} + \tfrac{1}{2}\delta^\top H \delta$ \\
$\hat{\Delta}_i$ & Attribution-patching estimate for component $i$ in a collection of candidate components \\
$c_i$ & Diagonal quadratic correction for component $i$, $\tfrac{1}{2}\delta_i^\top H_{ii}\delta_i$ \\
$S_{\mathrm{ok}}, S_{\mathrm{flag}}$ & Components with $\tilde R_i < \tau$ and $\tilde R_i \ge \tau$ under a threshold policy \\
$Q_{\mathrm{ok}}$ & Signed sum of uncorrected quadratic terms, $\sum_{i\in S_{\mathrm{ok}}} c_i$ \\
$\gamma$ & Cancellation ratio on uncorrected quadratic terms, $|Q_{\mathrm{ok}}|/\sum_{i\in S_{\mathrm{ok}}}|c_i|$ \\
\bottomrule
\end{tabular}
\end{table}

\section{Local smoothness examples of the Lipschitz-Hessian assumption}
\label{app:L3}

\subsection{Theoretical approximation}
\label{app:L3-theoretical}
We include representative constants only to clarify the scale of the cubic
remainder.  These values are not required by the main theorem, which only
assumes a finite $L_3$ on the segment of interest.

\paragraph{GeLU.}  $\text{GeLU}(z) = z\,\Phi(z)$ where $\Phi$ is the
standard normal CDF.  Differentiating gives
$\text{GeLU}'''(z) = z(z^2-4)\phi(z)$, where $\phi$ is the standard normal
PDF.  The stationary points of $|\text{GeLU}'''(z)|$ satisfy
$\text{GeLU}^{(4)}(z) = 0$, equivalently $z^4 - 7z^2 + 4 = 0$.  The maximizer
is therefore
$|z| = \sqrt{\frac{7-\sqrt{33}}{2}} \approx 0.792$, which gives
$\max_z |\text{GeLU}'''(z)| \approx 0.779$.

\paragraph{SiLU.}  $\text{SiLU}(z) = z \cdot \sigma(z)$ where $\sigma$ is
the logistic sigmoid.  The third derivative satisfies
$\max_z |\text{SiLU}'''(z)| \approx 0.30818$, attained near
$|z| \approx 1.032$ by numerical maximization.  Thus
$L_3^{\text{SiLU}} \approx 0.308$.  This value is only illustrative: the main
theorem itself does not rely on a closed-form constant.

\paragraph{Bilinear maps.}
If the perturbed variable enters bilinearly while the other argument is held
fixed, then the third derivative with respect to that variable is zero.  This
applies, for example, to perturbations of the value vector $V$ with fixed
attention weights, but it is \emph{not} a property of the full attention block
when $Q$ or $K$ also vary.

\paragraph{Softmax and normalization layers.}
Softmax is real analytic, so its third derivative is finite on every compact
subset of logit space.  LayerNorm and RMSNorm are smooth wherever the
normalization scale is bounded away from zero.  Because Proposition~\ref{thm:main}
requires only a finite constant on the single interpolation segment
$\{a+t\delta\}_{t\in[0,1]}$, we treat these contributions through the
segment-local $L_3$ rather than assert a universal transformer-wide bound.

\subsection{Empirical verification}
\label{app:l3-empirical}

Assumption~1 posits that the metric function's Hessian is $L_3$-Lipschitz
along the interpolation path.  We verify this empirically by computing
Hessian--vector products $H(a_c + t\delta)\cdot\delta$ at three points
($t = 0, 0.5, 1$) for each prompt $\times$ attention head, then estimating
\[
\hat{L}_3 = \max_{i}
\frac{\|H(a_c + t_{i+1}\delta)\cdot\delta - H(a_c + t_i\delta)\cdot\delta\|}
     {(t_{i+1} - t_i)\,\|\delta\|^2}\,.
\]
This is a lower bound on the true $L_3$ (probed along one direction only).
With $\hat{L}_3$ in hand, the cubic remainder bound becomes
$\hat{\alpha} = \hat{L}_3 \|\delta\|^3 / 6$, which we compare to the
actual residual $|f_{\mathrm{true}} - f_{\mathrm{HVP}}|$.

Table~\ref{tab:l3} reports the empirical $\hat{L}_3$ distribution and
bound-tightness statistics for GPT-2 (20 prompts $\times$ 144 heads =
2{,}880 records, 2{,}877 nontrivial).

\begin{table}[ht]
\centering
\small
\caption{Empirical $L_3$ estimation on GPT-2 (12L$\times$12H, 20 prompts).
$\hat{\alpha} = \hat{L}_3 \|\delta\|^3 / 6$ is the cubic remainder bound;
$|\text{res}|$ is the actual HVP residual $|f_{\text{true}} - f_{\text{HVP}}|$.
The bound holds when $\hat{\alpha} \geq |\text{res}|$.}
\label{tab:l3}
\begin{tabular}{lrrrr}
\toprule
& \textbf{Median} & \textbf{Mean} & \textbf{p90} & \textbf{Max} \\
\midrule
$\hat{L}_3$  & 0.0014 & 0.0169 & 0.0182 & 4.917 \\
$\hat{\alpha}$ (cubic bound) & 0.000108 & 0.00348 & --- & --- \\
$|\text{res}|$ (actual residual) & 0.000018 & 0.000920 & --- & --- \\
\midrule
\multicolumn{5}{l}{Bound holds: 2{,}370 / 2{,}877 = \textbf{82.4\%}} \\
\multicolumn{5}{l}{$\hat{\alpha}/|\text{res}|$: median $3.4\times$, mean $16.0\times$, p10 $0.20\times$} \\
\bottomrule
\end{tabular}
\vspace{10pt}
\end{table}

The bound holds for 82.4\% of component--prompt pairs, with a median
slack factor of $3.4\times$ (the bound is $3.4\times$ larger than the
actual residual).  The 17.6\% of violations are concentrated in Layer~0,
where $\hat{L}_3$ is an order of magnitude larger than mid-network layers
(median $0.065$ vs.\ $0.001$--$0.002$).  This is consistent with the
embedding layer's position-dependent structure producing sharper Hessian
variation.

Across layers, median $\hat{L}_3$ decreases monotonically from $0.065$
(L0) to $0.0001$ (L11), confirming the structural prediction that later
layers, with shorter downstream paths, exhibit smoother Hessian
landscapes.  This layer-depth gradient also explains why HVP correction
is most impactful for early- and mid-layer heads (where $L_3
\|\delta\|^3$ is large enough to matter) and nearly unnecessary for the
last few layers.

Table~\ref{tab:l3_gemma} reports the same analysis for Gemma-2-2B (26
layers $\times$ 8 heads = 208 heads, 20 prompts, 4{,}152 nontrivial
records).

\begin{table}[ht]
\centering
\small
\caption{Empirical $L_3$ estimation on Gemma-2-2B (26L$\times$8H, 20 prompts).
Same protocol as Table~\ref{tab:l3}.}
\label{tab:l3_gemma}
\begin{tabular}{lrrrr}
\toprule
& \textbf{Median} & \textbf{Mean} & \textbf{p90} & \textbf{Max} \\
\midrule
$\hat{L}_3$  & 0.0002 & 0.0014 & 0.0020 & 0.699 \\
$\hat{\alpha}$ (cubic bound) & 0.00403 & 0.1193 & --- & --- \\
$|\text{res}|$ (actual residual) & 0.00035 & 0.01104 & --- & --- \\
\midrule
\multicolumn{5}{l}{Bound holds: 3{,}945 / 4{,}152 = \textbf{95.0\%}} \\
\multicolumn{5}{l}{$\hat{\alpha}/|\text{res}|$: median $7.6\times$, mean $104\times$, p10 $2.7\times$} \\
\bottomrule
\end{tabular}
\vspace{10pt}
\end{table}

On Gemma-2-2B, the bound holds for \textbf{95.0\%} of pairs, stronger
than GPT-2's 82.4\%, despite much larger perturbation norms (median
$\|\delta\| = 4.81$ vs.\ $0.73$).  The median slack factor is $7.6\times$,
meaning the cubic bound is typically an order of magnitude larger than the
actual HVP residual.  Unlike GPT-2, Gemma-2-2B does not show a strong
monotonic layer-depth gradient in $\hat{L}_3$: the per-layer medians
hover around $0.0001$--$0.0009$ throughout the network, with occasional
spikes at L7 (max $0.42$) and L3 (max $0.70$).  This flatter profile is
consistent with Gemma-2's use of grouped-query attention and GeGLU, which
distribute nonlinearity more uniformly across layers than GPT-2's standard
architecture.

\section{Additional theoretical results and analysis}
\label{app:additional-theoretical-result}

\subsection{Multi-layer composition challenges}
\label{app:composition}

\paragraph{Why we do not state a global composition theorem.}
Composing per-layer curvature through a residual network introduces mixed
second derivatives and Jacobian growth factors.  Prior work on compositional
curvature bounds shows that such quantities can grow multiplicatively with
depth~\citep{entesari2024curvature}.  For this reason, our formal results are
stated at the level actually used by attribution patching and HVP correction: one component, one
input, and one perturbation segment.

\paragraph{Scope and computation.}
The bound in Proposition~1 is per-component and per-input.
Residual composition introduces mixed second derivatives and depth-dependent Jacobian factors~\citep{entesari2024curvature}, so the guarantee does not extend to a depth-uniform statement over the full transformer.
Computing $\tilde R$ requires only the Hessian-vector product $H\delta$, obtained by differentiating a scalar gradient--vector product rather than forming the full Hessian.
Empirically, $\tilde R$ serves as a practical diagnostic for flagging unreliable attributions (\S\ref{sec:discussion}).

\paragraph{Layerwise heuristic.}
If one additionally freezes downstream Hessians at the clean activations and
ignores mixed second derivatives, then the total attribution-patching error can be approximated
by summing the per-layer quadratic terms plus a cubic remainder.  We use this
only as intuition for the empirical depth trends; it is not a proved bound in
this paper.

\subsection{Error analysis of HVP correction}
\label{app:hvp_error}

\begin{corollary}[HVP correction error]
\label{thm:hvp_error}
After applying the HVP correction, the residual error is
\[
  E_{\text{hvp}} = \Delta - \hat\Delta_{\text{hvp}} = \Phi(\delta),
  \qquad |\Phi(\delta)| \leq \frac{L_3}{6}\|\delta\|^3.
\]
\end{corollary}

\begin{proof}
By definition,
\[
  \hat\Delta_{\text{hvp}} = \hat\Delta + \tfrac{1}{2}\delta^{\!\top}H\delta.
\]
Combining this with the Taylor decomposition in Eq.~\eqref{eq:error_decomp},
\[
  \Delta = \hat\Delta + \tfrac{1}{2}\delta^{\!\top}H\delta + \Phi(\delta),
\]
we obtain
\[
  \Delta - \hat\Delta_{\text{hvp}} = \Phi(\delta).
\]
The cubic remainder bound then follows immediately from the same
Hessian-Lipschitz assumption used in Proposition~\ref{thm:main}.
\end{proof}

\subsection{Selective-HVP workflow identities}
\label{app:selective_pipeline}

\begin{corollary}[Exact gap to full diagonal HVP]
\label{cor:selective_gap}
Let
\[
  \hat\Delta_{\mathrm{full}}
  :=
  \sum_{i=1}^n \hat\Delta_i + \sum_{i=1}^n c_i
\]
denote the estimate obtained by applying the diagonal HVP correction to every
independently patched component, and let
$E_{\mathrm{full}}^{\mathrm{ind}} := \sum_i \Delta_i - \hat\Delta_{\mathrm{full}}$.
Then
\begin{equation}
  E_{\mathrm{sel}}^{\mathrm{ind}} - E_{\mathrm{full}}^{\mathrm{ind}}
  =
  Q_{\mathrm{ok}}
  :=
  \sum_{i\in S_{\mathrm{ok}}} c_i.
  \label{eq:selective_gap}
\end{equation}
Consequently,
\begin{equation}
  \bigl|E_{\mathrm{sel}}^{\mathrm{ind}} - E_{\mathrm{full}}^{\mathrm{ind}}\bigr|
  \le
  \tau \sum_{i\in S_{\mathrm{ok}}} |\hat\Delta_i|.
  \label{eq:selective_gap_bound}
\end{equation}
\end{corollary}

\begin{proof}
By definition,
\[
  E_{\mathrm{full}}^{\mathrm{ind}}
  =
  \sum_{i=1}^n \Delta_i
  -
  \sum_{i=1}^n (\hat\Delta_i + c_i).
\]
Subtracting this from
$E_{\mathrm{sel}}^{\mathrm{ind}} = \sum_i \Delta_i - \hat\Delta_{\mathrm{sel}}$
and using
$\hat\Delta_{\mathrm{sel}} = \sum_i \hat\Delta_i + \sum_{i\in S_{\mathrm{flag}}} c_i$
gives
\[
  E_{\mathrm{sel}}^{\mathrm{ind}} - E_{\mathrm{full}}^{\mathrm{ind}}
  =
  \sum_{i=1}^n c_i - \sum_{i\in S_{\mathrm{flag}}} c_i
  =
  \sum_{i\in S_{\mathrm{ok}}} c_i
  =
  Q_{\mathrm{ok}}.
\]
Taking absolute values and using
$|c_i| = \tilde R_i |\hat\Delta_i| < \tau |\hat\Delta_i|$ for
$i\in S_{\mathrm{ok}}$ yields Eq.~\eqref{eq:selective_gap_bound}.
\end{proof}

\begin{proposition}[Sign-aware refinement for independent patching]
\label{prop:sign_aware_pipeline}
Define
$S_{\mathrm{ok}}^{+} := \{i \in S_{\mathrm{ok}} : c_i \neq 0\}$ and
$S_{0} := S_{\mathrm{ok}} \setminus S_{\mathrm{ok}}^{+}$.
For every $i \in S_{\mathrm{ok}}^{+}$ let $\alpha_i := |\Phi_i|/|c_i|$, and define
\[
  Q_{\mathrm{ok}}^{+} := \sum_{i\in S_{\mathrm{ok}}^{+}} c_i,
  \qquad
  \gamma := \frac{|Q_{\mathrm{ok}}^{+}|}
  {\sum_{i\in S_{\mathrm{ok}}^{+}} |c_i|},
  \qquad
  \bar\alpha := \max_{i\in S_{\mathrm{ok}}^{+}} \alpha_i,
\]
with the convention that the first term below is zero when
$S_{\mathrm{ok}}^{+} = \varnothing$.  Then
\begin{equation}
  \bigl|E_{\mathrm{sel}}^{\mathrm{ind}}\bigr|
  \le
  (\gamma + \bar\alpha)\,
  \tau \sum_{i\in S_{\mathrm{ok}}^{+}} |\hat\Delta_i|
  +
  \frac{1}{6}\sum_{i\in S_{\mathrm{flag}} \cup S_{0}}
  L_{3,i}\|\delta_i\|^3.
  \label{eq:sign_aware_pipeline}
\end{equation}
\end{proposition}

\begin{proof}
Decompose the selective residual as
\[
  E_{\mathrm{sel}}^{\mathrm{ind}}
  =
  \sum_{i\in S_{\mathrm{ok}}^{+}} (c_i + \Phi_i(\delta_i))
  +
  \sum_{i\in S_{0}} \Phi_i(\delta_i)
  +
  \sum_{i\in S_{\mathrm{flag}}} \Phi_i(\delta_i).
\]
For $i\in S_{\mathrm{ok}}^{+}$, the remainder bound and the definition of
$\alpha_i$ imply
$|\Phi_i(\delta_i)| \le \alpha_i |c_i| \le \bar\alpha |c_i|$.
Therefore
\begin{align*}
  \bigl|E_{\mathrm{sel}}^{\mathrm{ind}}\bigr|
  &\le
  \left| \sum_{i\in S_{\mathrm{ok}}^{+}} c_i \right|
  +
  \sum_{i\in S_{\mathrm{ok}}^{+}} |\Phi_i(\delta_i)|
  +
  \sum_{i\in S_{\mathrm{flag}} \cup S_0} |\Phi_i(\delta_i)| \\
  &\le
  |Q_{\mathrm{ok}}^{+}|
  +
  \bar\alpha \sum_{i\in S_{\mathrm{ok}}^{+}} |c_i|
  +
  \frac{1}{6}\sum_{i\in S_{\mathrm{flag}} \cup S_0}
  L_{3,i}\|\delta_i\|^3.
\end{align*}
Using
$|Q_{\mathrm{ok}}^{+}| = \gamma \sum_{i\in S_{\mathrm{ok}}^{+}} |c_i|$
and
$|c_i| = \tilde R_i |\hat\Delta_i| < \tau |\hat\Delta_i|$
for $i\in S_{\mathrm{ok}}^{+}$ yields Eq.~\eqref{eq:sign_aware_pipeline}.
\end{proof}

\section{Omitted proofs}
\label{app:proofs}

\subsection{Proof of Proposition~\ref{thm:main} (Local attribution-patching error bound via the reliability score)}

\begin{proof}
From the Taylor expansion with integral remainder,
\[
  M(a + \delta) = M(a) + \nabla_a M \cdot \delta
  + \tfrac{1}{2}\delta^{\!\top} H \delta + \Phi(\delta),
\]
where the remainder takes the standard integral form:
\[
  \Phi(\delta) = \int_0^1 (1-t)\,
  \delta^{\!\top}\!\bigl[\nabla_a^2 M(a + t\delta) - H\bigr]\delta\, dt.
\]
Under Assumption~\ref{ass:local_smooth},
\[
  \|\nabla_a^2 M(a + t\delta) - H\|_{\mathrm{op}}
  \leq L_3 \cdot t\|\delta\|,
\]
and therefore
\begin{align}
  |\Phi(\delta)| &\leq \int_0^1 (1-t) \cdot L_3 \, t
  \|\delta\|^3\, dt
  = L_3 \|\delta\|^3 \int_0^1 t(1-t)\, dt \notag \\
  &= L_3 \|\delta\|^3 \cdot \frac{1}{6}
  = \frac{L_3 \|\delta\|^3}{6}\,.
  \label{eq:remainder_bound}
\end{align}
Here $\int_0^1 t(1-t)\,dt = \frac{1}{2} - \frac{1}{3} = \frac{1}{6}$.

Defining $Q = \tfrac{1}{2}\delta^{\!\top}\!H\delta$, we have $E = Q + \Phi$
and the reverse triangle inequality gives
\[
  \bigl||E| - |Q|\bigr| \le |\Phi|.
\]
With
\[
  \alpha
  =
  \frac{L_3\|\delta\|^3}{3|\delta^{\!\top}\!H\delta|}
  =
  \frac{|\Phi|_{\max}}{|Q|},
\]
we obtain
\[
  \bigl||E| - |Q|\bigr| \le \alpha |Q|.
\]
Dividing by $|\hat\Delta|$ and noting $|Q|/|\hat\Delta| = \tilde{R}$ gives
\[
  \left|
    \frac{|E|}{|\hat\Delta|} - \tilde R
  \right|
  \le
  \alpha \tilde R,
\]
which is the first claimed inequality.  If additionally $\alpha < 1$, then
\[
  (1 - \alpha)|Q| \leq |E| \leq (1 + \alpha)|Q|,
\]
and dividing again by $|\hat\Delta|$ yields the two-sided inequality in
Proposition~\ref{thm:main}.
\end{proof}

\subsection{Proof of Corollary~\ref{thm:hvp_error_main} (Residual error after HVP correction)}
\begin{proof}
By definition,
\[
  \hat\Delta_{\text{hvp}} = \hat\Delta + \tfrac{1}{2}\delta^{\!\top}H\delta.
\]
Combining this with the Taylor decomposition in Eq.~\eqref{eq:error_decomp},
\[
  \Delta = \hat\Delta + \tfrac{1}{2}\delta^{\!\top}H\delta + \Phi(\delta),
\]
we obtain
\[
  \Delta - \hat\Delta_{\text{hvp}} = \Phi(\delta).
\]
The cubic remainder bound then follows immediately from the same
Hessian-Lipschitz assumption used in Proposition~\ref{thm:main}.
\end{proof}

\subsection{Proof of Proposition~\ref{thm:selective_pipeline_main} (Selective-HVP pipeline guarantee)}

\begin{proof}
For each independently patched component $i$, the single-component Taylor
decomposition gives
\[
  \Delta_i - \hat\Delta_i = c_i + \Phi_i(\delta_i),
  \qquad
  |\Phi_i(\delta_i)| \le \frac{L_{3,i}}{6}\|\delta_i\|^3.
\]
Subtracting the selective estimate
\[
  \hat\Delta_{\mathrm{sel}}
  =
  \sum_{i=1}^n \hat\Delta_i + \sum_{i\in S_{\mathrm{flag}}} c_i
\]
from the sum of true independent effects gives
\[
  E_{\mathrm{sel}}^{\mathrm{ind}}
  =
  \sum_{i=1}^n (\Delta_i - \hat\Delta_i) - \sum_{i\in S_{\mathrm{flag}}} c_i
  =
  \sum_{i\in S_{\mathrm{ok}}} (c_i + \Phi_i(\delta_i))
  +
  \sum_{i\in S_{\mathrm{flag}}} \Phi_i(\delta_i).
\]
By the triangle inequality,
\[
  |E_{\mathrm{sel}}^{\mathrm{ind}}|
  \le
  \sum_{i\in S_{\mathrm{ok}}} |c_i|
  +
  \sum_{i=1}^n |\Phi_i(\delta_i)|.
\]
For every $i \in S_{\mathrm{ok}}$, the threshold definition gives
$\tilde R_i = |c_i|/|\hat\Delta_i| < \tau$, hence
\[
  |c_i| < \tau |\hat\Delta_i|.
\]
For the remainders,
\[
  |\Phi_i(\delta_i)| \le \frac{L_{3,i}}{6}\|\delta_i\|^3
  \qquad \text{for all } i.
\]
Summing these bounds yields
\[
  |E_{\mathrm{sel}}^{\mathrm{ind}}|
  \le
  \tau \sum_{i\in S_{\mathrm{ok}}} |\hat\Delta_i|
  +
  \frac{1}{6}\sum_{i=1}^n L_{3,i}\|\delta_i\|^3,
\]
which is exactly Eq.~\eqref{eq:selective_pipeline_bound}.
\end{proof}

\subsection{Proof of Corollary~\ref{cor:selective_gap}}
\begin{proof}
By definition,
\[
  E_{\mathrm{full}}^{\mathrm{ind}}
  =
  \sum_{i=1}^n \Delta_i
  -
  \sum_{i=1}^n (\hat\Delta_i + c_i).
\]
Subtracting this from
$E_{\mathrm{sel}}^{\mathrm{ind}} = \sum_i \Delta_i - \hat\Delta_{\mathrm{sel}}$
and using
$\hat\Delta_{\mathrm{sel}} = \sum_i \hat\Delta_i + \sum_{i\in S_{\mathrm{flag}}} c_i$
gives
\[
  E_{\mathrm{sel}}^{\mathrm{ind}} - E_{\mathrm{full}}^{\mathrm{ind}}
  =
  \sum_{i=1}^n c_i - \sum_{i\in S_{\mathrm{flag}}} c_i
  =
  \sum_{i\in S_{\mathrm{ok}}} c_i
  =
  Q_{\mathrm{ok}}.
\]
Taking absolute values and using
$|c_i| = \tilde R_i |\hat\Delta_i| < \tau |\hat\Delta_i|$ for
$i\in S_{\mathrm{ok}}$ yields Eq.~\eqref{eq:selective_gap_bound}.
\end{proof}

\section{Implementation details}

\subsection{Setup}
\label{app:setup}

\paragraph{Models}
We evaluate five model families at different scales as shown in Table~\ref{tab:model-comparison}. GPT-2~\citep{radford2019gpt2}, Pythia (410M, 2.8B)~\citep{biderman2023pythia}, Qwen2.5-1.5B~\citep{qwen2.5}, Gemma (2B, 9B)~\citep{gemma2024}, and Llama-3.1-8B~\citep{llama3_2024}. All models are hooked with TransformerLens~\citep{nanda2022transformerlens}. We analyze attention heads in all models, pre-activation MLP neurons in GeLU models, and MLP layer outputs in Qwen.
\begin{table}[ht]
\centering
\caption{Comparison of model architectures used in our experiments.}
\label{tab:model-comparison}
\begin{tabular}{lccccc}
\toprule
Model & Layers & Heads & $d_{\text{model}}$ & $d_{\text{mlp}}$ & Activation \\
\midrule
GPT-2 Small~\citep{radford2019gpt2} & 12 & 12 & 768 & 3072 & GeLU \\
Pythia-410M-deduped~\citep{biderman2023pythia} & 24 & 16 & 1024 & 4096 & GeLU \\
Pythia-2.8B-deduped~\citep{biderman2023pythia} & 32 & 32 & 2560 & 10240 & GeLU \\
Qwen2.5-1.5B~\citep{qwen2.5} & 28 & 12 & 1536 & 8960 & SwiGLU \\
Gemma-2-2B~\citep{gemma2024} & 26 & 8 & 2304 & 9216 & GeGLU \\
Gemma-2-9B~\citep{gemma2024} & 42 & 16 & 3584 & 14336 & GeGLU \\
Llama-3.1-8B~\citep{llama3_2024} & 32 & 32 & 4096 & 14336 & SwiGLU \\
\bottomrule
\end{tabular}
\end{table}
For factual completion and Greater-Than tasks, corruptions are generated by replacing the token at position~3 with a uniformly sampled vocabulary token. For IOI, we adopt the standard name-swap corruption introduced by \citep{wang2023ioi}. In all settings, $\delta$ denotes the induced activation perturbation and $M$ denotes the log-probability of the correct next token.

Our primary metric is \emph{top-5 relative error}: the mean absolute error of a method's scores on the five components with largest activation-patching effect, normalized by the ground-truth score range.

\paragraph{Tasks}
We evaluate on the three established circuit benchmarks together with a broader factual-completion setting:
\begin{itemize}
    \item \textbf{IOI:} For the IOI circuit ranking experiment, we follow the setup of~\citet{wang2023ioi}: 50 IOI prompts of the form ``When Mary and John went to the store, John gave a drink to'', with the indirect object (Mary) as the target.  We compute attribution-patching and HVP-corrected attributions for all $12 \times 12 = 144$ attention heads, rank them by magnitude, and compare the resulting rankings against the activation-patching ground truth using top-$k$ overlap and Spearman correlation.
    \item \textbf{Greater-Than:} For the Greater-Than circuit~\citep{hanna2023greater}, we use 200 prompts of the form ``The war lasted from the year 17\{XX\} to the year 17'' and measure whether the model assigns higher probability to years greater than XX.  All 144 attention heads in GPT-2 Small are ranked by attribution magnitude.  Top-$k$ overlap is computed against the activation patching ground truth. 
    \item \textbf{Factual completion.} For generic model sweeps, we use factual next-token prediction prompts and measure the effect of component interventions on the correct-token log-probability.
\end{itemize}

Unless otherwise noted, generic sweeps use 20 factual prompts. We additionally scale selected experiments to 55 prompts for the Pythia-410M attention-head setting and 35 prompts for the Qwen2.5-1.5B attention-head and MLP-output settings. IOI uses 50 templated prompts and Greater-Than uses 200 prompts.

Our model-task selection is guided by the theoretical prediction of \S\ref{sec:method}, which predicts that attribution-patching error increases when the quadratic form $\delta^{\!\top}\!H\delta$ dominates the first-order approximation.  Pythia-410M on IOI represents a predicted high-error regime in which the local Taylor approximation breaks down, whereas GPT-2 Small on IOI represents a low-error regime where standard attribution patching is already accurate. Gemma-2-2B factual completion lies between these extremes. Additional model-task pairs provide broader cross-architecture coverage (Table~\ref{tab:master}).

\paragraph{Evaluation protocol.}
Unless otherwise stated, we report 95\% confidence intervals computed from 1{,}000 bootstrap resamples over prompts. Each resample draws prompts with replacement and aggregates over all components within the sampled prompt set. We report the 2.5th and 97.5th percentiles of the bootstrap distribution. Circuit-ranking metrics on IOI and Greater-Than are reported as point estimates over the benchmark datasets. We measure the following component types:
\begin{itemize}
    \item \textbf{Pre-activation MLP neurons:} \texttt{hook\_pre} activations in GeLU-based models, with 256 sampled neurons per layer.
    \item \textbf{Qwen MLP outputs:} \texttt{hook\_mlp\_out} activations, corresponding to full $d_{\text{model}}$-dimensional layer outputs.
    \item \textbf{Post-activation MLP neurons:} \texttt{hook\_post} activations.
    \item \textbf{Attention heads:} \texttt{hook\_z} activations, i.e., head outputs before the output projection.
    \item \textbf{Residual stream:} \texttt{hook\_resid\_post} activations over the full residual vector.
\end{itemize}

\paragraph{Baselines.}
We compare against AtP*~\citep{kramar2024atp} and GIM~\citep{edin2025gim}. Since AtP* does not provide an official implementation, we reimplemented the method from the paper description. AtP* combines Q/K-linearization, which linearizes the softmax backward pass, with GradDrop, which suppresses gradient contributions from components whose activation region changes under perturbation. Because these corrections are specific to attention mechanisms, comparisons with AtP* are restricted to attention-head experiments.

For GIM, we follow the public implementation, modifying the softmax backward pass to remove the self-repair gradient term responsible for systematic underestimation. Both baselines are evaluated on the same prompt sets and component collections as the corresponding HVP experiments in Table~\ref{tab:comparison}.

Hessian--vector products are computed in PyTorch via
\texttt{torch.autograd.grad} with \texttt{create\_graph=True} on the first
backward pass, so the gradient can itself be differentiated on the second
backward pass.  We never explicitly form the Hessian; only the product
$H\delta$ is computed, which is what keeps the correction tractable at 8B
scale.  Gradients are kept in float32 throughout (mixed precision elsewhere).

\begin{table}[ht]
\caption{Comparison of attribution-patching correction methods on Pythia-410M using the matched
20-prompt comparison protocol (bootstrap 95\% CIs for error columns).  HVP
provides the largest reduction at $3\times$ cost.}
\label{tab:comparison}
\centering
\small
\begin{tabular}{lrrrrr}
\toprule
Method & \multicolumn{2}{c}{Attn.\ heads (\%)} & \multicolumn{2}{c}{Pre-act neurons (\%)} & Cost \\
\cmidrule(lr){2-3} \cmidrule(lr){4-5}
 & Error & Reduct.\ & Error & Reduct.\ & \\
\midrule
Attribution patching & 4.4 [3.6, 5.2] & --- & 35.2 [31.1, 39.3] & --- & $1\times$ \\
AtP*              & 2.9 [2.2, 3.6] & 34.1 & 33.8 [29.8, 37.8] & 4.0 & $1\times$ \\
GIM               & 2.6 [1.9, 3.3] & 40.9 & 31.4 [27.5, 35.3] & 10.8 & $1\times$ \\
HVP (ours)        & 1.0 [0.7, 1.3] & \textbf{77.4} & 11.7 [9.4, 14.0] & \textbf{66.7} & $3\times$ \\
Act.\ patching    & 0 & 100 & 0 & 100 & $N\times$ \\
\bottomrule
\end{tabular}
\end{table}

\subsection{Metric}
\label{app:metric_identity}
For the generic model sweeps, the metric $M$ is the log-probability of the
correct next token at the evaluation position.

\paragraph{Identity.} The near-identical values for log probability, probability, and cross-entropy in Table~\ref{tab:metrics} are an empirical observation rather than a formal invariance theorem.  Over the perturbation magnitudes used in
our experiments, these metrics are smooth monotone re-parameterizations of the same target logit score and are locally well approximated by affine maps, so their relative attribution-patching errors agree to within numerical noise. We verified this
directly: the maximum observed deviation between log-probability and probability relative errors is 0.03\%.

\paragraph{Sensitivity.} Table~\ref{tab:metrics} shows that HVP correction is robust across five different metric functions, with reduction ranging from 62\% to 72\%. The raw-logit metric yields the largest reduction (72\%), as expected: the final softmax introduces additional curvature that the second-order correction only partially captures. This confirms that the HVP correction is not specific to the choice of output metric.

\begin{table}[ht]
\caption{HVP correction across metric functions (pre-activation neurons,
Pythia-410M, layer 23, 128 neurons, bootstrap 95\% CIs).  Log probability,
probability, and cross-entropy yield near-identical errors in this
perturbation regime (Appendix~\ref{app:metric_identity}); the maximum
deviation between probability and log-probability relative errors is 0.03\%.}
\label{tab:metrics}
\centering
\small
\begin{tabular}{lrrr}
\toprule
Metric & Attrib.\ Pat.\ Err.\ (\%) & HVP Err.\ (\%) & Reduction (\%) \\
\midrule
Raw logit        & 33.3 [27.9, 38.7] & 9.4 [6.8, 12.0]  & \textbf{71.7} [64.2, 79.2] \\
Logit difference & 38.9 [33.1, 44.7] & 11.8 [8.9, 14.7] & \textbf{69.6} [62.1, 77.1] \\
Log probability  & 40.9 [35.0, 46.8] & 15.4 [12.0, 18.8] & \textbf{62.3} [54.8, 69.8] \\
Probability      & 40.9 [35.0, 46.8] & 15.4 [12.0, 18.8] & \textbf{62.4} [54.9, 69.9] \\
Cross-entropy    & 40.9 [35.0, 46.8] & 15.4 [12.0, 18.8] & \textbf{62.3} [54.8, 69.8] \\
\bottomrule
\end{tabular}
\end{table}

\subsection{Per-head IG vs.\ all-at-once IG (EAP-IG)}
\label{app:ig_distinction}

The ``IG'' baseline in Table~\ref{tab:master} is \emph{per-head} Integrated
Gradients: for each attention head $h$, we interpolate only $h$'s activation
from clean to corrupted while holding all other activations at clean, and
integrate the gradient along this isolated path using $S$ quadrature points.
By the fundamental theorem of calculus, this recovers the head $h$'s exact activation-patching effect as the same quantity the ground truth measures, which
is why we treat it as the accuracy ceiling.  The per-prompt cost is $S$ forward
+ backward passes per head, totalling $S \times N_{\text{heads}}$ per prompt.
For Llama-3.1-8B ($32 \times 32 = 1{,}024$ heads) at $S{=}10$, this is
$10{,}240$ backward passes per prompt, yielding the $\sim$25 GPU-day estimate
in \S\ref{sec:broad_gains}.

All-at-once IG (termed EAP-IG in \citet{syed2024eap} and used by the MIB
benchmark~\citep{mueller2025mib}) interpolates \emph{all} source activations
simultaneously: at step $t$, every source is set to
$(1{-}t)\,a_{\text{clean}} + t\,a_{\text{corrupt}}$, and one forward +
backward pass at each of $S$ steps yields edge-level attributions for the
entire graph.  The total cost is $S$ forward + backward passes per prompt
regardless of the number of heads, making it feasible at any scale.  However,
EAP-IG computes a fundamentally different quantity: because all sources move
together, the resulting per-edge scores conflate cross-component interaction
effects.  EAP-IG is a first-order method (gradient only, no Hessian
correction) and answers the question ``how does each edge contribute along the
joint interpolation path?'' rather than ``what is each head's isolated causal
effect?''

HVP's cost structure resembles the all-at-once approach rather than the per-head approach.  The key computation $\delta^{\!\top}\!H\delta$ via \texttt{torch.autograd.grad} with \texttt{create\_graph=True} hooks all components simultaneously and produces a global second-order correction in two backward passes per prompt.  The per-component corrections are then extracted
as coordinate projections of the full Hessian--vector product.  This batching
is why HVP at cost~2 scales to 8B+ parameters, whereas per-head IG at cost
$S {\times} N_{\text{heads}}$ does not.

\subsection{Compute and wall-clock timing}
\label{app:wallclock}

All experiments ran on NVIDIA L40S GPUs.  Pythia-410M: $\sim$6 hours including the 55-prompt attention expansion.  Pythia-2.8B: $\sim$8 hours.  GPT-2 Small: $\sim$1 hour for the generic sweeps plus $\sim$30 minutes for IOI and $\sim$4 hours for the 200-prompt Greater-Than benchmark.  Qwen2.5-1.5B: $\sim$4 hours including the 35-prompt attention and MLP expansions.  AtP*/GIM comparison: $\sim$3 hours.  Total compute is approximately 25 hours of single-GPU time.  No experiment required more than a single GPU.

Measured on a single NVIDIA L40S GPU.  For Pythia-410M: standard attribution patching takes 0.8s per prompt (all components), HVP correction takes 2.2s ($2.8\times$ overhead, not exactly $3\times$, due to computation graph caching). Pythia-2.8B: attrib.\ pat.\ 3.1s, HVP 8.9s ($2.9\times$).  GPT-2 Small: attrib.\ pat.\ 0.4s, HVP 1.1s ($2.8\times$).

Table~\ref{tab:wallclock} reports wall-clock attribution time for the
edge-level circuit-discovery pipeline on two representative settings, measured
on a single NVIDIA L40S GPU with 100 evaluation examples.  The
``selective'' row estimates the Screen--Flag--Fix workflow: run EAP to
obtain edge scores, flag components with $\tilde{R} > \tau$, then apply
HVP corrections only to flagged edges.

\begin{table}[ht]
\centering
\small
\caption{Wall-clock attribution time (seconds) on L40S.  MIB evaluation
time (model reloading + 10 sparsity sweeps) adds $\sim$12--24\,min and
is method-independent.  The selective pipeline (Screen--Flag--Fix) applies
HVP only to flagged edges ($\tau{=}0.3$), combining EAP cost with a
small HVP overhead.}
\label{tab:wallclock}
\begin{tabular}{llrrr}
\toprule
\textbf{Model} & \textbf{Method} & \textbf{Edges} & \textbf{Attrib.\ (s)} & \textbf{Overhead vs.\ EAP} \\
\midrule
\multicolumn{5}{@{}l}{\emph{GPT-2 (12L, 32{,}491 edges)}} \\
& EAP              & 32{,}491 & 1.6  & $1.0\times$ \\
& HVP ($K{=}1$)    & 32{,}491 & 5.8  & $3.6\times$ \\
& MS-HVP ($K{=}5$) & 32{,}491 & 4.5  & $2.8\times$ \\
& EAP-IG (inputs)  & 32{,}491 & 4.9  & $3.1\times$ \\
& Selective HVP ($\tau{=}0.3$, 7\% flagged) & 32{,}491 & $\sim$1.9 & $\sim 1.2\times$ \\
\midrule
\multicolumn{5}{@{}l}{\emph{Qwen2.5-0.5B (24L, 179{,}749 edges)}} \\
& EAP              & 179{,}749 & 5.9  & $1.0\times$ \\
& HVP ($K{=}1$)    & 179{,}749 & 6.1  & $1.0\times$ \\
& MS-HVP ($K{=}5$) & 179{,}749 & 8.8  & $1.5\times$ \\
& Selective HVP ($\tau{=}0.3$, 15\% flagged) & 179{,}749 & $\sim$6.0 & $\sim 1.0\times$ \\
\bottomrule
\end{tabular}
\vspace{10pt}
\end{table}

The key takeaway: on both models, the selective pipeline adds $\leq 20\%$
wall-clock overhead over raw EAP while still capturing the most important
corrections.  Full HVP adds $1.0$--$3.6\times$ overhead -- modest in
absolute terms (seconds, not minutes) because the HVP backward pass reuses
the same computation graph as the EAP forward pass.  The dominant cost in a
full MIB evaluation is the faithfulness sweep (reloading the model at each
sparsity level), which is method-independent and takes 12--24\,min.

\section{Full results of matched-compute comparisons}
\label{app:comparisons}

\subsection{Comparison to integrated gradients.}
\label{app:hvp_vs_ig}
\begin{table}[ht]
\caption{Matched-cost paired bootstrap: MS-HVP $K{=}5$ vs.\ IG $S{=}10$
(cost~10 each).  $\Delta$ is MS-HVP minus IG (negative favours MS-HVP).
$p$-values are one-sided paired bootstrap (10{,}000 resamples).}
\label{tab:matched_cost}
\centering
\small
\begin{tabular}{lrrrrl}
\toprule
Model / Task & MS-HVP (\%) & IG (\%) & $\Delta$ (pp) & $p$ & Verdict \\
\midrule
Pythia-410M IOI    & 18.03 & 17.26 & $+$0.77 & 0.620   & tie \\
GPT-2 IOI          & \textbf{2.97}  & 3.30  & $-$0.34 & $<$0.001 & \textbf{win} \\
Gemma-2-2B factual & 40.66 & 40.57 & $+$0.10 & 0.696   & tie \\
Pythia-410M GT     & 21.47 & 24.62 & $-$3.15 & 0.347   & tie \\
GPT-2 GT           & 23.33 & 32.98 & $-$9.65 & 0.183   & tie \\
Gemma-2-2B IOI     & 6.58  & 6.63  & $-$0.05 & 0.320   & tie \\
Qwen2.5-1.5B IOI   & \textbf{2.94}  & 3.16  & $-$0.22 & 0.041   & \textbf{win} \\
Pythia-410M factual & 64.41 & 63.91 & $+$0.50 & 0.732   & tie \\
Qwen2.5-1.5B factual & 32.06 & 31.95 & $+$0.10 & 0.668 & tie \\
\bottomrule
\end{tabular}
\end{table}

\begin{figure}
    \centering
    \includegraphics[width=\linewidth]{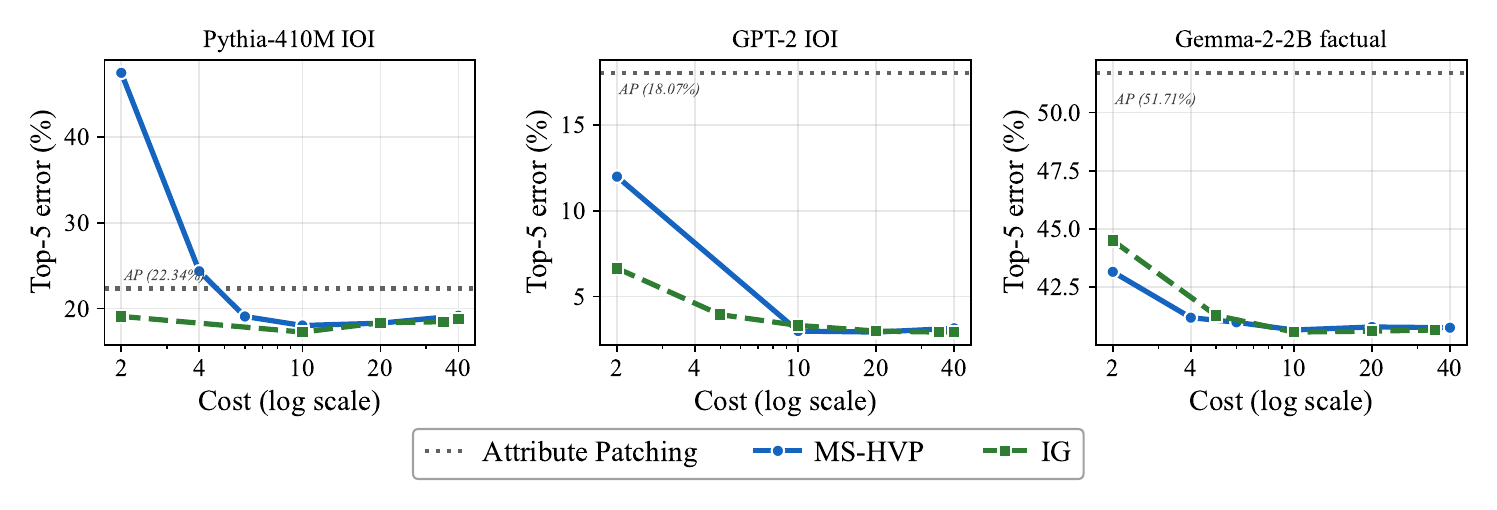}
    \caption{\textbf{Cost--accuracy tradeoff: MS-HVP vs.\ integrated gradients.}
    Top-$5$ relative error as a function of compute cost (backward passes per component) on three representative regimes: pathological high-curvature (Pythia-410M IOI), low-error clean (GPT-2 IOI), and moderate-error factual (Gemma-2-2B factual).
    Horizontal dotted lines denote attribution patching without correction.
    MS-HVP matches or exceeds IG at comparable compute across all regimes, while achieving substantially better compute efficiency in the pathological setting.}
    \label{fig:ig_vs_hvp}
\end{figure}

Integrated gradients (IG) provides the strongest existing accuracy baseline but scales linearly with the number of interpolation steps. Table~\ref{tab:matched_cost} reports paired-bootstrap $p$-values comparing
MS-HVP $K{=}5$ (cost~10) against IG $S{=}10$ (cost~10) on all nine non-frontier
tasks.  Figure~\ref{fig:ig_vs_hvp} compares the compute--accuracy tradeoff between IG and MS-HVP across representative low-, moderate-, and high-curvature regimes. At matched cost~10, MS-HVP ($K{=}5$) and IG ($S{=}10$) are statistically tied on seven of nine tasks under paired-bootstrap testing, while MS-HVP significantly outperforms IG on GPT-2 IOI ($p < 0.001$) and Qwen2.5-1.5B IOI ($p = 0.041$). Thus, iterative quadratic correction achieves parity with path integration on most tasks despite using only local second-order information.

The difference becomes more pronounced in the pathological high-curvature regime. On Pythia-410M IOI, MS-HVP with $K{=}5$ (cost~10) outperforms IG with $S{=}35$ (cost~35), achieving $18.03\%$ versus $18.47\%$ top-5 error ($\Delta = -0.44$~pp, $p = 0.022$). At approximately matched cost~40, aggregating MS-HVP estimates across $K \in \{5,10,20\}$ via a per-head median further improves performance to $17.57\%$, surpassing IG's best result of $18.77\%$ ($p < 0.001$). These results indicate that multi-step quadratic correction can recover the benefits of dense path integration while requiring substantially less computation.

Importantly, the practical scaling behavior differs sharply between the methods. IG requires $S$ full backward passes per component and rapidly becomes infeasible at 8B scale, whereas MS-HVP remains tractable because each step reuses local curvature information. Consequently, MS-HVP is not only competitive at matched cost but also extends naturally to model sizes where IG cannot practically run.

\subsection{Comparison to Integrated Hessians}
\label{app:hvp_vs_ih}

Table~\ref{tab:ih_detail} compares Integrated Hessians (IH) against MS-HVP across all three foreground tasks.  Both IH-PI (path-integrated) and IH-DR (double Riemann) underperform MS-HVP by $2.0$--$7.2$ percentage points.  The gap arises because IH's $\alpha\beta$ weighting yields a $1/4$ coefficient on $\delta^{\!\top}\!H\delta$ instead of the $1/2$ needed for Taylor correction.

\begin{table}[ht]
\caption{Integrated Hessians vs.\ MS-HVP on the three foreground tasks. IH-PI uses $S$ path interpolation steps; IH-DR uses an $S{\times}M$ double Riemann grid.  All $p$-values are paired bootstrap vs.\ attribution patching.}
\label{tab:ih_detail}
\centering
\small
\begin{tabular}{lrrr}
\toprule
Method & Cost & Top-5 (\%) & $\Delta$ vs.\ attrib.\ pat.\ (pp) \\
\midrule
\multicolumn{4}{@{}l}{\emph{Pythia-410M IOI}} \\
MS-HVP $K{=}5$ & 10 & \textbf{18.03} & $-$4.31 \\
IH-PI $S{=}5$  & 10 & 22.22 & $-$0.12 \\
IH-DR $3{\times}3$ & 18 & 20.51 & $-$1.83 \\
IH-DR $5{\times}5$ & 50 & 19.99 & $-$2.35 \\
\midrule
\multicolumn{4}{@{}l}{\emph{GPT-2 IOI}} \\
MS-HVP $K{=}5$ & 10 & \textbf{2.97} & $-$15.10 \\
IH-PI $S{=}5$  & 10 & 10.15 & $-$7.92 \\
IH-DR $3{\times}3$ & 18 & 9.81 & $-$8.26 \\
\midrule
\multicolumn{4}{@{}l}{\emph{Gemma-2-2B factual}} \\
MS-HVP $K{=}5$ & 10 & \textbf{40.66} & $-$11.05 \\
IH-PI $S{=}5$  & 10 & 44.41 & $-$7.31 \\
IH-DR $3{\times}3$ & 18 & 44.46 & $-$7.25 \\
\bottomrule
\end{tabular}
\end{table}

\subsection{Comparison to GIM}
\label{app:hvp_vs_gim}
Table~\ref{tab:gim_recovery} shows that GIM underperforms attribution patching
on 17 of 18 task$\times K$ head-recovery settings, with losses of up to
$21.8$~pp (Llama factual @5).  The sole positive entry is Gemma-2-2B IOI @5
($+0.4$~pp), which is within noise.

\begin{table}[ht]
\caption{GIM vs.\ attribution patching: top-$K$ head overlap across nine tasks.
$\Delta$ is GIM minus attribution patching (negative = GIM worse).}
\label{tab:gim_recovery}
\centering
\small
\begin{tabular}{lrrrrrr}
\toprule
Model / Task & Pat.@5 & GIM@5 & $\Delta$@5 & Pat.@10 & GIM@10 & $\Delta$@10 \\
\midrule
Pythia-410M IOI      & 81.6 & 72.4 & $-$9.2  & 69.8 & 63.6 & $-$6.2 \\
GPT-2 IOI            & 87.6 & 74.8 & $-$12.8 & 93.2 & 84.2 & $-$9.0 \\
Gemma-2-2B IOI       & 86.8 & 87.2 & $+$0.4  & 87.8 & 83.6 & $-$4.2 \\
Llama-3.1-8B IOI     & 77.6 & 67.6 & $-$10.0 & 86.2 & 77.4 & $-$8.8 \\
Pythia-410M GT       & 74.3 & 69.7 & $-$4.6  & 62.0 & 55.9 & $-$6.1 \\
GPT-2 GT             & 69.5 & 67.2 & $-$2.3  & 50.2 & 47.3 & $-$2.9 \\
Pythia-410M factual  & 31.3 & 29.1 & $-$2.2  & 29.8 & 24.9 & $-$4.9 \\
Gemma-2-2B factual   & 47.3 & 29.1 & $-$18.2 & 49.3 & 32.2 & $-$17.1 \\
Llama-3.1-8B factual & 46.2 & 24.4 & $-$21.8 & 40.0 & 26.9 & $-$13.1 \\
\bottomrule
\end{tabular}
\end{table}

\subsection{Recommendation map}
\label{app:recommendations}

Table~\ref{tab:recommendations} distills the findings of matched-compute experiments into a deployment guide organized by scale and component type.

\begin{table}[ht]
\centering
\caption{Recommendation map.  Choose method by scale and budget.  Costs are
backward passes per head per prompt; $K$ is the MS-HVP sub-step count.}
\label{tab:recommendations}

\setlength{\tabcolsep}{5pt}
\begin{tabular}{@{}l l p{4.2cm}@{}}
\toprule
Scale / budget & Recommended method & Rationale \\
\midrule
Frontier ($\geq 8$B)          & Std HVP ($K{=}1$, cost $2$)              & IG/IH infeasible at this scale \\
Smaller, cost-tolerant        & MS-HVP $K{=}5$ or IG $S{=}10$           & Exchangeable at matched cost (7 ties, 2 wins at cost~10) \\
Smaller, accuracy-priority    & MS-HVP $K{=}10$ (cost $20$)             & Empirical knee on pathological tasks \\
Budget-constrained            & Selective HVP, $\tau \in [0.1, 0.3]$    & Tunable cost-accuracy tradeoff \\
Pre-act.\ MLPs (any)          & MS-HVP $K{=}5$                          & Strongest gains (56--67\%) \\
Post-act.\ MLPs               & Attrib.\ pat.\ suffices                 & Near-linear regime; 2--7\% gain \\
\bottomrule
\end{tabular}
\end{table}

\section{Additional experimental results}
\label{app:additional-experimental-results}

\subsection{Full per-model results}
\label{app:full-per-model-results}
Table~\ref{tab:full_results} provides the complete per-model breakdown for the
main generic sweeps, including residual-stream and post-activation components,
plus the Gemma-2-2B attention-head run.
\begin{table}[ht]
\caption{Full HVP correction results across all component types and models.
\textbf{Attrib.\ Pat.\ Err.} and \textbf{HVP Err.} are overall medians across available
component-prompt records.  \textbf{Reduction} uses the prompt-level median
percentage decrease in relative attribution-patching error, with bootstrap 95\% CIs over
prompts.}
\label{tab:full_results}
\centering
\small
\begin{tabular}{llrrrr}
\toprule
Component & Model & Attrib.\ Pat.\ Err.\ (\%) & HVP Err.\ (\%) & Reduction (\%) & $N$ \\
\midrule
Attention heads   & Pythia-410M  & 4.1 & 1.1 & \textbf{72.0} [69.4, 74.7] & 20,993 \\
Pre-act neurons   & Pythia-410M  & 35.2 [31.1, 39.3] & 11.7 [9.4, 14.0] & \textbf{66.7} [60.2, 72.8] & 5,871 \\
Residual stream   & Pythia-410M  & 45.4 [39.5, 51.3] & 13.6 [10.4, 16.8] & \textbf{70.0} [63.7, 75.8] & 480 \\
Post-act neurons  & Pythia-410M  & 4.3 [3.7, 4.9] & 4.0 [3.4, 4.6] & 7.0 [2.1, 11.9] & 2,943 \\
\midrule
Pre-act neurons   & Pythia-2.8B  & 35.4 [31.6, 39.2] & 15.6 [13.5, 17.7] & \textbf{55.8} [50.1, 61.5] & 27,393 \\
Post-act neurons  & Pythia-2.8B  & 6.2 [5.4, 7.0] & 6.1 [5.3, 6.9] & 2.2 [0.5, 3.9] & 27,393 \\
\midrule
Attention heads   & GPT-2 Small  & 5.1 [4.1, 6.1] & 1.3 [0.9, 1.7] & \textbf{74.5} [68.3, 79.9] & 5,760 \\
Pre-act neurons   & GPT-2 Small  & 32.8 [28.4, 37.2] & 12.1 [9.6, 14.6] & \textbf{63.1} [56.8, 69.4] & 4,320 \\
Residual stream   & GPT-2 Small  & 42.1 [36.8, 47.4] & 14.2 [11.1, 17.3] & \textbf{66.3} [59.5, 72.4] & 240 \\
\midrule
Attention heads   & Qwen2.5-1.5B & 5.4 & 0.5 & \textbf{90.4} [89.5, 91.8] & 11,680 \\
MLP layer output  & Qwen2.5-1.5B & 45.6 & 15.7 & \textbf{68.0} [59.4, 71.3] & 980 \\
\midrule
Attention heads   & Gemma-2-2B   & 7.4 & 0.8 & \textbf{90.0} [88.7, 90.6] & 19,749 \\
\bottomrule
\end{tabular}
\end{table}

\subsection{Selective-HVP cost / coverage tradeoff}
\label{app:selective_tradeoff}

Figure~\ref{fig:selective_tradeoff} plots the empirical operating curve of the
practical workflow from Section~\ref{sec:discussion}: run attribution patching broadly, threshold
$\tilde{R}$, and apply HVP only to the flagged subset.  The left panel shows
how many components are corrected as the threshold $\tau$ varies; the right
panel shows the resulting reduction in median attribution-patching error, with the dotted lines
marking the full-HVP ceiling for each model.  At the fixed threshold
$\tau=0.3$, the selective pipeline flags only 7.0\% of components on
Pythia-410M, 14.9\% on Qwen2.5-1.5B, and 19.0\% on Gemma-2-2B, while still
reducing median attribution-patching error by 10.1\%, 23.9\%, and 30.1\%, respectively.  The
corresponding cancellation ratios $\gamma \in \{0.641, 0.429, 0.197\}$ remain
well below 1, consistent with the sign-aware refinement from
Appendix~\ref{app:selective_pipeline}: the uncorrected quadratic terms
partially cancel, so the realized selective gap is less pessimistic than a
worst-case absolute-sum bound.

\begin{figure*}[t]
\centering
\includegraphics[width=0.92\textwidth]{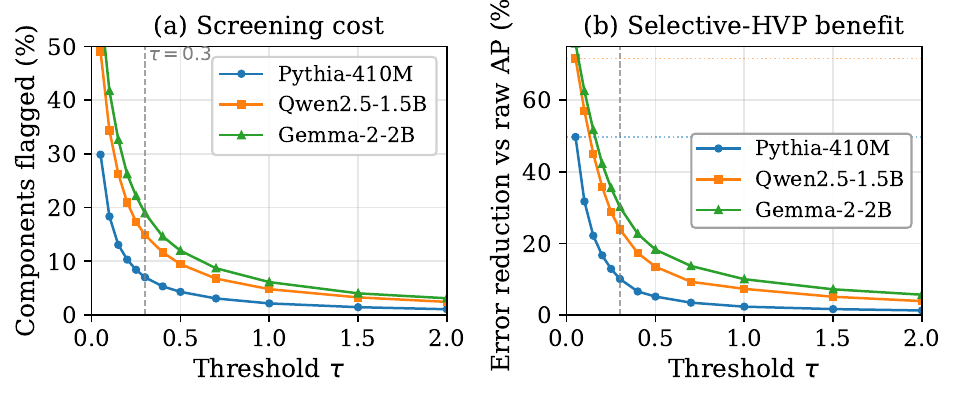}
\caption{Selective-HVP operating curve on three representative models.  Left:
fraction of components flagged for HVP correction as the threshold $\tau$ on
$\tilde{R}$ varies.  Right: reduction in median attribution-patching error relative to raw attribution patching;
dotted lines show the full-HVP ceiling.  At $\tau=0.3$, the selective pipeline
flags only 7.0\%, 14.9\%, and 19.0\% of components on Pythia-410M,
Qwen2.5-1.5B, and Gemma-2-2B, while still reducing median attribution-patching error by 10.1\%,
23.9\%, and 30.1\%, respectively.}
\label{fig:selective_tradeoff}
\end{figure*}

\subsection{Corruption-type robustness}
\label{app:corruption_robustness}

Table~\ref{tab:corruption_robustness} evaluates HVP under three corruption
styles: the main random-token corruption, cross-prompt resample corruption, and
zero corruption.  The top block reports auxiliary model-family checks; the
bottom block tests the three foreground tasks from Table~\ref{tab:master}
directly.  On clean tasks (GPT-2 IOI, Gemma-2-2B factual), HVP reduction
remains strong (79--84\%).  On the pathological task (Pythia-410M IOI), Std HVP
reduction is weaker (36--50\%), consistent with the Taylor-radius violation
identified in \S\ref{sec:broad_gains} -- MS-HVP $K{\geq}3$ would be needed here.
The starred Pythia-1B random-token row is an outlier with unusually large
perturbation norms.

\begin{table}[ht]
\centering
\caption{Corruption robustness.  \emph{Top block:} auxiliary models from
initial robustness check.  \emph{Bottom block:} foreground tasks from
Table~\ref{tab:master}.  Reduction uses the prompt-level median relative error
reduction.  The starred Pythia-1B random-token row is an outlier with unusually
large perturbation norms and only 320 records.}
\label{tab:corruption_robustness}
\small
\begin{tabular}{llrrr}
\toprule
\textbf{Model / Task} & \textbf{Corruption} & \textbf{AP med.\ err (\%)} & \textbf{HVP med.\ err (\%)} & \textbf{Reduction (\%)} \\
\midrule
\multicolumn{5}{@{}l}{\emph{Auxiliary models}} \\
Gemma-2-2B (aux) & random-token & 7.4 & 0.8 & \textbf{90.0} \\
Gemma-2-2B (aux) & resample & 16.1 & 3.5 & \textbf{79.6} \\
Gemma-2-2B (aux) & zero & 18.3 & 3.8 & \textbf{78.2} \\
Pythia-1B & resample & 12.8 & 1.4 & \textbf{88.8} \\
Pythia-1B & zero & 13.8 & 1.6 & \textbf{87.8} \\
Pythia-1B$^\ast$ & random-token & 42.6 & 34.0 & 12.1 \\
\midrule
\multicolumn{5}{@{}l}{\emph{Foreground tasks}} \\
GPT-2 IOI & zero & 13.2 & 2.2 & \textbf{83.5} \\
GPT-2 IOI & resample & 4.9 & 0.9 & \textbf{82.5} \\
Pythia-410M IOI & zero & 9.4 & 4.7 & 50.2 \\
Pythia-410M IOI & resample & 13.3 & 8.5 & 36.2 \\
Gemma-2-2B factual & zero & 19.2 & 4.1 & \textbf{78.7} \\
Gemma-2-2B factual & resample & 17.0 & 3.5 & \textbf{79.4} \\
\bottomrule
\end{tabular}
\end{table}

\subsection{Semantic (entity-swap) corruption}
\label{app:semantic_corruption}

To verify that HVP's correction generalizes beyond random-token perturbations,
we run per-head attribution patching and HVP correction on 49
entity-swap factual prompts (e.g., ``The Eiffel Tower is located in''
$\to$ ``The Colosseum is located in'').  These perturbations are
multi-token, semantically coherent, and produce correlated activation shifts across layers -- a more naturalistic corruption regime than the
single-position random-token replacements used in the main text.

\begin{table}[ht]
\centering
\small
\caption{Entity-swap (semantic) corruption results.  Both models use 46
entity-swap prompt pairs (e.g., ``The Eiffel Tower is located in'' $\to$
``The Colosseum is located in'').  MAE is the mean absolute error vs.\
ground-truth activation patching across all nontrivial head $\times$ prompt
entries.}
\label{tab:semantic_corruption}
\begin{tabular}{llrrr}
\toprule
\textbf{Model} & \textbf{Method} & \textbf{MAE} & \textbf{Relative to AP} \\
\midrule
GPT-2 (144 heads, 7{,}056 records) & AP  & 0.00340 & $1.00\times$ \\
 & HVP & \textbf{0.00156} & $0.46\times$ \\
 & \multicolumn{3}{l}{\emph{Error reduction: 54.1\%}} \\
\midrule
Pythia-410M (384 heads, 18{,}816 records) & AP  & 0.00272 & $1.00\times$ \\
 & HVP & \textbf{0.00112} & $0.41\times$ \\
 & \multicolumn{3}{l}{\emph{Error reduction: 58.8\%}} \\
\bottomrule
\end{tabular}
\vspace{10pt}
\end{table}

Both models show substantial error reduction under entity-swap corruption
(54\% for GPT-2, 59\% for Pythia-410M), confirming that HVP's second-order
correction is not specific to random-token perturbations.  Entity-swap
corruptions produce larger, more structured $\delta$ vectors (because
multiple token positions change), yet the Hessian--vector product still
captures the dominant curvature.  The slightly stronger reduction on
Pythia-410M (58.8\% vs.\ 54.1\%) is consistent with this model's deeper
architecture (24 layers vs.\ 12), which amplifies cross-layer nonlinear
interactions that the second-order term corrects.

\subsection{Stability of HVP correction estimates}
\label{app:stability}

Figure~\ref{fig:stability} plots the \emph{aggregate} HVP error-reduction
estimate as the number of evaluation prompts increases.  Unlike the main-text
tables, which report the prompt-level median reduction with prompt-bootstrap
confidence intervals, this figure tracks the single aggregate statistic obtained
from progressively larger prompt prefixes.  Both architectures converge
rapidly: the Pythia-410M estimate stabilizes in the low- to mid-70s after
roughly 10--15 prompts and ends at 73.3\% for 55 prompts, while the
Qwen2.5-1.5B estimate remains near 90\% throughout and ends at 90.2\% for 35
prompts.  This supports the claim that the scaled headline numbers are not
driven by a lucky small-sample subset.

\begin{figure}[ht]
\centering
\includegraphics[width=0.6\textwidth]{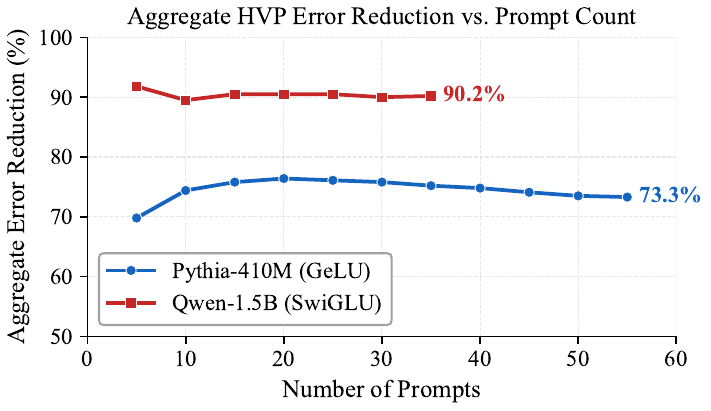}
\caption{Stability of the \emph{aggregate} HVP error-reduction estimate as the
number of evaluation prompts grows.  Both curves flatten after roughly 10--15
prompts.  The final aggregate estimates are 73.3\% for Pythia-410M (55 prompts)
and 90.2\% for Qwen2.5-1.5B (35 prompts).}
\label{fig:stability}
\end{figure}

\subsection{Ranking performance}
\label{app:ranking}
Table~\ref{tab:ranking} supplements the error-reduction analysis with rank-correlation metrics computed from existing per-head attribution data. For each task and method, we aggregate per-head scores (mean $|\text{score}|$ across prompts), rank all $N$ heads, and compare to the ground-truth ranking via Kendall $\tau$ (global rank correlation) and NDCG@$K$ (quality of top-$K$ recovery, $K \in \{5, 10, 20\}$).

\begin{table}[ht]
\centering
\small
\caption{Ranking metrics across tasks and methods.  $\tau$: Kendall rank
correlation with ground truth (higher $=$ better).  NDCG@$K$: Normalized
Discounted Cumulative Gain at rank $K$ (higher $=$ better).  All $p < 10^{-6}$.
Best per task in \textbf{bold}.}
\label{tab:ranking}
\begin{tabular}{llrrrr}
\toprule
Model / Task & Method & $\tau$ & NDCG@5 & NDCG@10 & NDCG@20 \\
\midrule
\multicolumn{6}{@{}l}{\emph{GPT-2 (12L$\times$12H = 144 heads)}} \\
IOI & AP              & 0.568 & 1.000 & 1.000 & 0.968 \\
    & HVP             & 0.567 & 1.000 & 1.000 & 0.968 \\
    & IH (3$\times$3) & 0.552 & 1.000 & 0.934 & 1.000 \\
    & GIM             & 0.552 & 1.000 & 0.934 & 1.000 \\
GT  & AP              & 0.687 & 1.000 & 0.936 & 0.930 \\
    & HVP             & \textbf{0.689} & 1.000 & 0.936 & \textbf{0.931} \\
\midrule
\multicolumn{6}{@{}l}{\emph{Pythia-410M (24L$\times$16H = 384 heads)}} \\
IOI & AP              & 0.350 & \textbf{1.000} & 0.936 & 0.934 \\
    & HVP             & \textbf{0.352} & 0.699 & 0.890 & 0.935 \\
    & IH-PI ($S{=}5$) & 0.350 & \textbf{1.000} & 0.936 & 0.934 \\
    & GIM             & 0.350 & \textbf{1.000} & 0.936 & 0.934 \\
GT  & AP              & 0.582 & \textbf{1.000} & \textbf{1.000} & 0.935 \\
    & HVP             & \textbf{0.585} & 0.869 & \textbf{1.000} & 0.935 \\
\midrule
\multicolumn{6}{@{}l}{\emph{Gemma-2-2B (26L$\times$8H = 208 heads)}} \\
IOI      & AP  & 0.650 & 0.869 & 1.000 & 1.000 \\
         & HVP & \textbf{0.654} & \textbf{1.000} & 1.000 & 1.000 \\
factual  & AP  & 0.719 & \textbf{1.000} & 0.791 & 0.934 \\
         & HVP & \textbf{0.734} & 0.869 & \textbf{0.863} & 0.932 \\
         & GIM & 0.719 & \textbf{1.000} & 0.791 & 0.934 \\
\midrule
\multicolumn{6}{@{}l}{\emph{Llama-3.1-8B (32L$\times$32H = 1{,}024 heads)}} \\
IOI & AP & 0.409 & 0.723 & 0.931 & 0.933 \\
\bottomrule
\end{tabular}
\vspace{10pt}
\end{table}

Attribution patching's global ranking is already strong: Kendall $\tau$
ranges from $0.35$ (Pythia-410M IOI, 384 heads) to $0.72$ (Gemma-2-2B
factual).  HVP matches or slightly improves $\tau$ on most tasks (e.g.,
$+0.004$ on Gemma-2-2B IOI, $+0.015$ on Gemma-2-2B factual), with
negligible change on GPT-2 IOI ($-0.001$, within bootstrap noise).
This is consistent with the main-text findings: HVP's gains concentrate at
ranking boundaries where circuit-membership decisions are made, not in
global reorderings.

On Pythia-410M IOI, HVP improves $\tau$ ($+0.002$) and NDCG@20 ($+0.001$)
but reduces NDCG@5 (from $1.00$ to $0.70$).  This reflects a known
trade-off in the pathological regime: Std HVP ($K{=}1$) overshoots on a
few high-$\tilde R$ heads (Table~\ref{tab:master}), reranking them away
from the top 5; MS-HVP $K{\geq}3$ resolves this
(\S\ref{sec:comparisons}).

NDCG@5 is $\geq 0.87$ in all non-pathological settings.  On Llama-3.1-8B
IOI (1{,}024 heads), we report the AP baseline only ($\tau = 0.409$).

\subsection{Case study: L4H11 mis-ranking in GPT-2 IOI}
\label{app:case_study}

We illustrate the practical impact of HVP correction with a concrete
wrong-ranking example from the GPT-2 IOI circuit (50 prompts, 144 heads).
Head L4H11 is a known \emph{duplicate-token head} in the IOI circuit and
ranks 7th by true activation-patching effect ($f_{\mathrm{true}} = 0.765$).
Attribution patching severely underestimates its effect ($\hat{\Delta}_{\mathrm{AP}} = 0.174$, a $4.4\times$ underestimate), placing it at rank~27 -- outside any reasonable
top-$K$ circuit.  HVP partially recovers the true score
($\hat{\Delta}_{\mathrm{HVP}} = 0.404$), promoting L4H11 to rank~12.

This single correction cascades into improved circuit recovery across all
$K$ thresholds on this task:

\begin{table}[ht]
\centering
\caption{Top-$K$ overlap with ground-truth ranking on GPT-2 IOI (50 prompts).
HVP recovers L4H11 and other underestimated heads, improving overlap at every
threshold.}
\label{tab:ioi_topk_overlap}
\small
\begin{tabular}{lcccc}
\toprule
& Top-10 & Top-15 & Top-20 & Top-26 \\
\midrule
Attribution patching & 80\% & 93\% & 95\% & 96\% \\
HVP-corrected & \textbf{90\%} & \textbf{100\%} & \textbf{100\%} & \textbf{100\%} \\
\bottomrule
\end{tabular}
\vspace{10pt}
\end{table}

The L4H11 example is representative of a broader pattern: attribution patching
tends to underestimate heads whose clean activation has large norm (producing
large $\|\delta\|$ under name-swap corruption), precisely the regime where the
second-order correction is most needed.  The reliability score flags this head
with $\tilde{R} = |0.404 - 0.174|/|0.174| = 1.32 \gg 0.3$, correctly
identifying it as a candidate for HVP refinement.

Figure~\ref{fig:heatmap_ioi} visualizes the full $12 \times 12$ attribution
landscape.  In the ground-truth panel~(a), L4H11 is clearly one of the
brightest components; in the attribution-patching panel~(b), it nearly vanishes; in
the HVP-corrected panel~(c), it is partially restored.  Blue outlines mark
the 23 known IOI circuit heads from \citet{wang2023ioi}.

\begin{figure*}[ht]
\centering
\includegraphics[width=\textwidth]{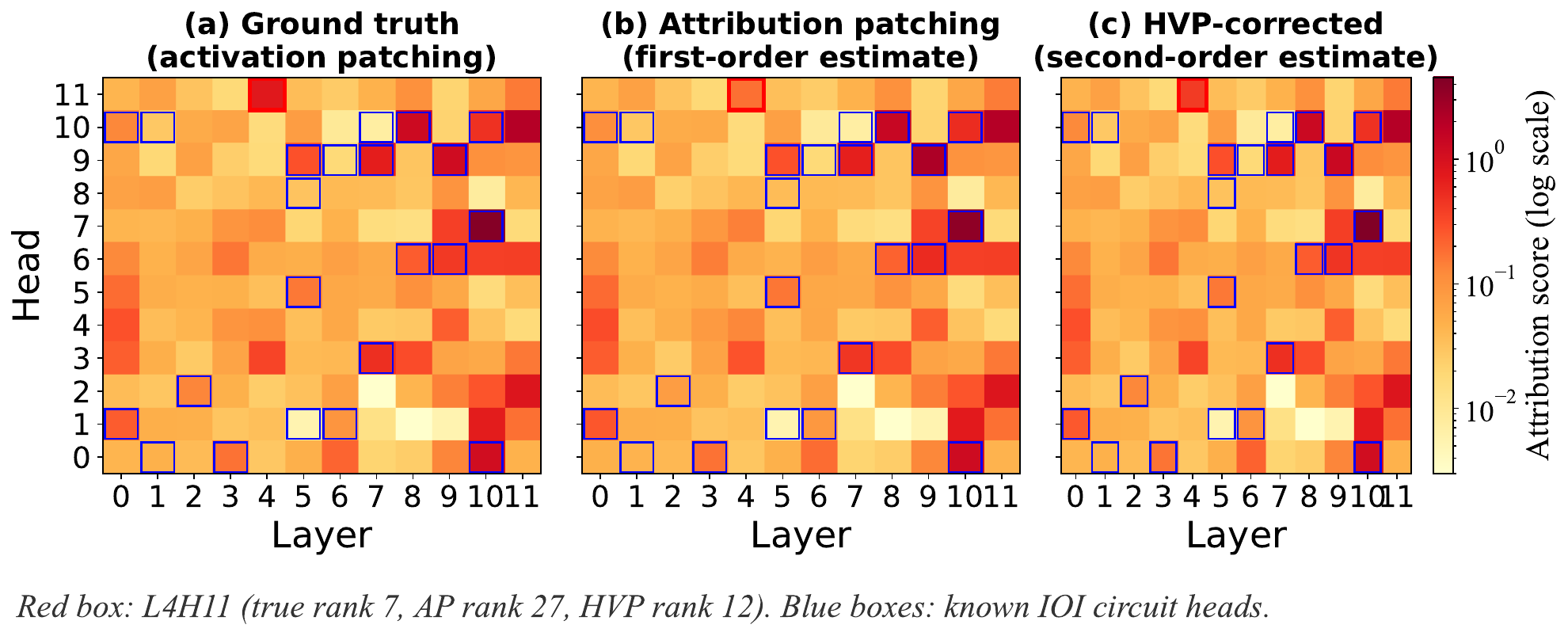}
\caption{Layer$\times$head attribution heatmaps for GPT-2 IOI (log scale).
\textbf{(a)}~Ground truth (activation patching).
\textbf{(b)}~Attribution patching (first-order).
\textbf{(c)}~HVP-corrected (second-order).
Red box: L4H11 (true rank~7, attribution-patching rank~27, HVP rank~12).
Blue boxes: known IOI circuit heads.}
\label{fig:heatmap_ioi}
\end{figure*}

Figure~\ref{fig:circuit_ioi} provides an alternative view, plotting only the
18 circuit heads grouped by functional role, with circle area proportional to
attribution score.  L4H11's circle is barely visible in the
attribution-patching panel but grows substantially in the HVP panel.

Finally, Figure~\ref{fig:attention_ioi} shows the attention patterns of three
key circuit heads on a representative IOI prompt
(\textit{``When John and Mary went to the store, Mary gave a drink to''}).
L10H7 (name mover) attends strongly to ``John'' (the indirect object) at the
prediction position; L9H9 (negative name mover) also attends to ``John'' but
with a negative contribution; and L4H11 (duplicate token) exhibits a
characteristic diagonal pattern, detecting repeated tokens.  The nonlinear
interaction between L4H11's duplicate-token detection and the name-swap
corruption explains why its attribution-patching estimate is so far from the
true activation-patching effect.
\begin{figure*}[t]
\centering
\includegraphics[width=\textwidth]{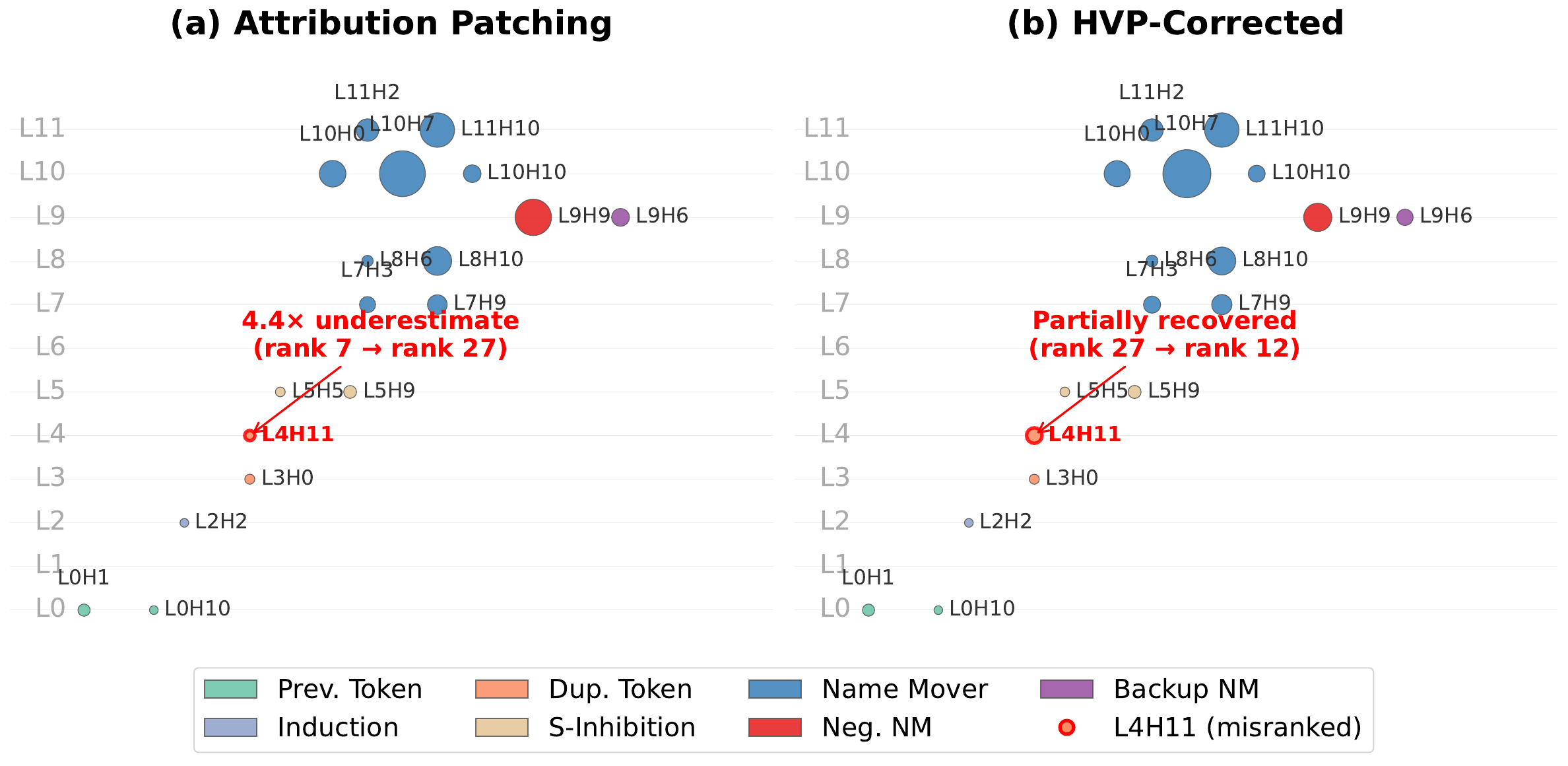}
\caption{IOI circuit heads grouped by functional role (GPT-2 Small).  Circle
area $\propto$ attribution score.  \textbf{(a)}~Attribution patching scores.
\textbf{(b)}~HVP-corrected scores.  L4H11 (red outline, duplicate-token head)
is severely underestimated by attribution patching ($0.17$ vs.\ true $0.77$)
and partially recovered by HVP ($0.40$).}
\label{fig:circuit_ioi}
\end{figure*}

\begin{figure*}[t]
\centering
\includegraphics[width=\textwidth]{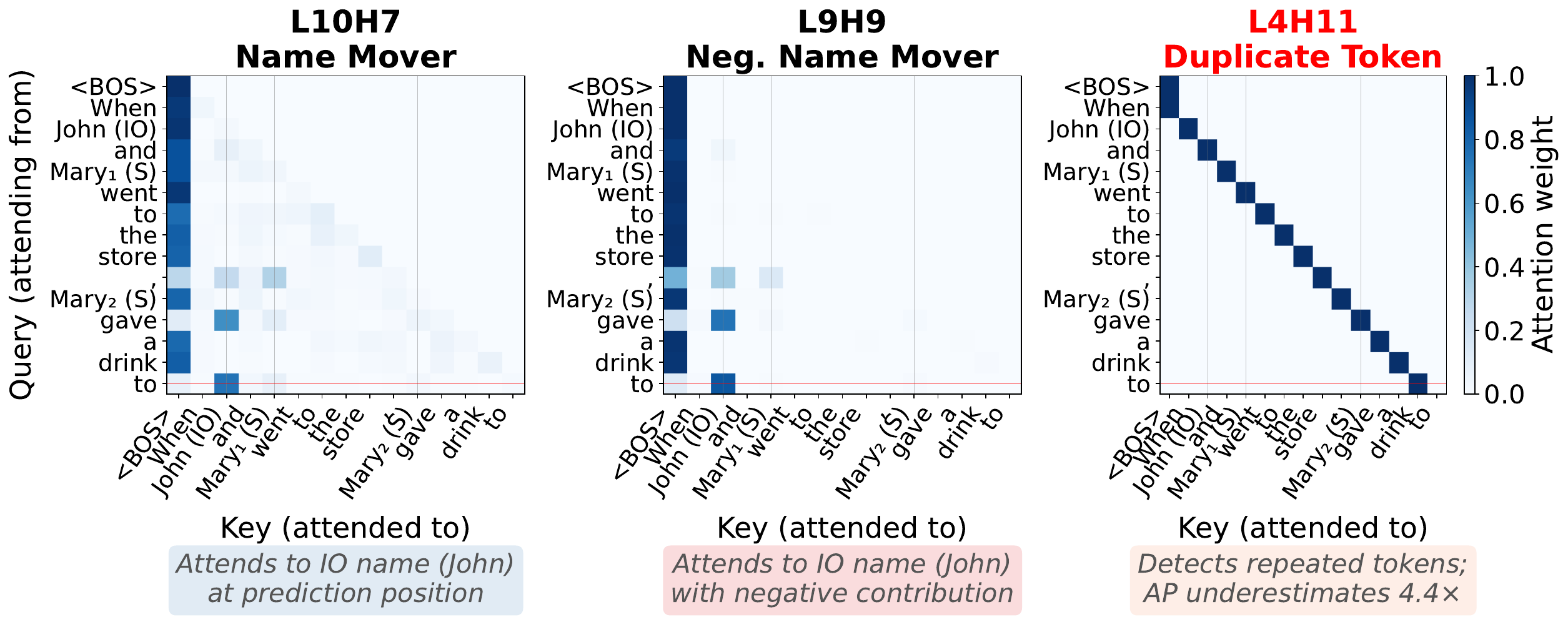}
\caption{Attention patterns of three key IOI circuit heads on a representative
prompt.  Each panel shows the full query$\times$key attention matrix; the red
horizontal line marks the prediction position (last token ``to'').
\textbf{Left:} L10H7 (name mover) attends 74\% to ``John'' (IO).
\textbf{Center:} L9H9 (negative name mover) attends 85\% to ``John''.
\textbf{Right:} L4H11 (duplicate token, red title) detects repeated tokens
via a diagonal pattern.  Attribution patching underestimates L4H11 by
$4.4\times$.}
\label{fig:attention_ioi}
\end{figure*}

\subsection{Circuit-recovery head overlap}
\label{app:head_recovery}
Table~\ref{tab:head_gt_merged} reports top-$K$ head overlap against activation-patching ground truth on the Greater-Than circuit. 
\begin{itemize}
    \item \textbf{GPT-2 Greater-Than:} MS-HVP $K{=}5$ pushes top-5 recovery to $83.2\%$ (vs.\ $70.1\%$ for Attribution Patching).
    \item \textbf{Pythia-410M Greater-Than:} Top-5 overlap improves from $76.0\%$ to $80.4\%$ with standard HVP, while GIM degrades to $69.7\%$.
\end{itemize}

\begin{table}[ht]
    \caption{Top-$K$ head recovery on Greater-Than tasks ($N{=}200$). Overlap is the fraction of ground-truth heads recovered at each $K$. Best results per model are in \textbf{bold}.}
    \label{tab:head_gt_merged}
    \centering
    \begin{tabular}{lrrrrrr}
    \toprule
    & \multicolumn{3}{c}{GPT-2} & \multicolumn{3}{c}{Pythia-410M} \\
    \cmidrule(lr){2-4} \cmidrule(lr){5-7}
    Method & @5 & @10 & @20 & @5 & @10 & @20 \\
    \midrule
    Attribution patching & 70.1\% & 51.7\% & 49.1\% & 76.0\% & 63.4\% & 49.8\% \\
    Std HVP ($K{=}1$) & 72.5\% & 52.4\% & 49.1\% & \textbf{80.4\%} & \textbf{65.0\%} & 51.0\% \\
    IG $S{=}10$ & 73.0\% & 52.4\% & 49.4\% & 80.0\% & 64.4\% & 51.0\%  \\
    \textbf{MS-HVP $K{=}5$} & \textbf{83.2\%} & \textbf{59.4\%} & \textbf{53.8\%} & 79.6\% & 64.8\% & \textbf{51.3\%} \\
    GIM & 68.2\% & 49.7\% & 46.8\% & 69.7\% & 55.9\% & 42.9\% \\
    \bottomrule
    \end{tabular}
\end{table}

\subsection{MIB benchmark comparison}
\label{app:mib}

We evaluate whether the component-level corrections translate to
circuit-level faithfulness using the official MIB evaluator on the public
test split~\citep{mueller2025mib}. CPR uses signed, high-value edge selection
and reports area under the performance-recovery curve (higher is better).
CMD uses absolute, high-magnitude edge selection and reports area from one in
the metric-deviation curve (lower is better). We compare EAP, input-level
EAP-IG, and exact isolated-edge HVP and MS-HVP scores. We report the six
non-Llama settings for which exact isolated-edge evaluations were completed.

\begin{table}[!ht]
\centering
\small
\caption{MIB circuit-faithfulness results using the official public-test
protocol. HVP and MS-HVP use exact isolated-edge scores. Best result for each
model--task and metric is in \textbf{bold}.}
\label{tab:mib}
\begin{tabular}{llrr}
\toprule
Model / Task & Method & CPR $\uparrow$ & CMD $\downarrow$ \\
\midrule
GPT-2 $\times$ IOI
 & EAP             & 1.1940 & 0.0352 \\
 & EAP-IG (inputs) & 1.8329 & 0.0321 \\
 & HVP             & 2.1062 & \textbf{0.0245} \\
 & MS-HVP $K{=}5$  & \textbf{2.2364} & 0.0253 \\
\midrule
Qwen2.5-0.5B $\times$ IOI
 & EAP             & 0.2678 & 0.0407 \\
 & EAP-IG (inputs) & 1.6360 & 0.0192 \\
 & HVP             & 1.5741 & 0.0181 \\
 & MS-HVP $K{=}5$  & \textbf{1.7648} & \textbf{0.0176} \\
\midrule
Qwen2.5-0.5B $\times$ MCQA
 & EAP             & 0.8562 & 0.0857 \\
 & EAP-IG (inputs) & 1.1454 & 0.0703 \\
 & HVP             & 0.9523 & 0.0597 \\
 & MS-HVP $K{=}5$  & \textbf{1.3253} & \textbf{0.0270} \\
\midrule
Gemma-2-2B $\times$ IOI
 & EAP             & 1.2848 & \textbf{0.0445} \\
 & EAP-IG (inputs) & \textbf{3.2920} & 0.0541 \\
 & HVP             & 1.3463 & 0.0509 \\
 & MS-HVP $K{=}5$  & 2.1451 & 0.0544 \\
\midrule
Gemma-2-2B $\times$ MCQA
 & EAP             & 1.4693 & 0.0705 \\
 & EAP-IG (inputs) & 1.5812 & 0.0779 \\
 & HVP             & 1.7965 & 0.0924 \\
 & MS-HVP $K{=}5$  & \textbf{1.9673} & \textbf{0.0509} \\
\midrule
Gemma-2-2B $\times$ ARC-Easy
 & EAP             & 1.4337 & 0.0368 \\
 & EAP-IG (inputs) & 1.5979 & 0.0433 \\
 & HVP             & 1.7635 & 0.0238 \\
 & MS-HVP $K{=}5$  & \textbf{1.9406} & \textbf{0.0168} \\
\bottomrule
\end{tabular}
\end{table}

MS-HVP obtains both higher CPR and lower CMD than EAP and EAP-IG in five
of the six evaluated settings. Relative to exact HVP, MS-HVP improves CPR
in all six settings and CMD in four. Gemma-2-2B IOI is the main exception,
showing that the circuit-level gains are substantial but not uniform.

\subsection{SAE feature-level HVP correction}
\label{app:sae_feature}

To test whether HVP's second-order correction extends beyond attention
heads to \emph{sparse autoencoder (SAE) features}, we decompose the
residual stream at two layers of GPT-2 Small (layers 5 and 9) using
pretrained SAEs from \citet{bloom2024saetrainingcodebase} (JumpReLU,
768$\to$24{,}576 features) and compute per-feature attribution patching
and HVP correction on the IOI task (50 prompts).

For each prompt, we identify all SAE features with nontrivial
ground-truth activation-patching effect ($|f_{\mathrm{true}}| > 10^{-8}$),
yielding $\sim$2{,}500 features per layer.  We then compare per-feature AP
and HVP estimates against the ground truth.

\begin{table}[!ht]
\centering
\small
\caption{SAE feature-level attribution accuracy on GPT-2 IOI (50 prompts).
Each row reports the mean absolute error (MAE) and Pearson correlation
of per-feature AP and HVP estimates vs.\ ground-truth activation patching.}
\label{tab:sae_feature}
\begin{tabular}{lrrrrr}
\toprule
\textbf{Layer} & \textbf{Features} & \textbf{AP MAE} & \textbf{HVP MAE} & \textbf{Reduction} & \textbf{Corr.\ (AP $\to$ HVP)} \\
\midrule
Layer 5 (resid\_pre) & 2{,}476 & 0.000048 & \textbf{0.000009} & \textbf{81.5\%} & 0.9999 $\to$ 1.0000 \\
Layer 9 (resid\_pre) & 2{,}500 & 0.000669 & \textbf{0.000097} & \textbf{85.5\%} & 0.9996 $\to$ 1.0000 \\
\bottomrule
\end{tabular}
\end{table}

HVP reduces feature-level MAE by 81--86\%, with correlation improving
from 0.9996--0.9999 to effectively 1.0000.  Layer~9 shows stronger
absolute errors (MAE $6.7 \times 10^{-4}$ for AP vs.\ $4.8 \times 10^{-5}$
at Layer~5), consistent with more features having large effects in later
layers, but HVP's relative improvement is even larger (85.5\% vs.\ 81.5\%).

This confirms that HVP's second-order correction is not specific to
attention heads: it generalizes to any decomposition of the residual
stream, including SAE feature directions.  The practical implication is
that HVP can improve the accuracy of feature-level circuit discovery
workflows that use SAEs to identify causally important features.


\clearpage
\pagenumbering{gobble}
\newpage
\section*{NeurIPS Paper Checklist}

\begin{enumerate}

\item {\bf Claims}
    \item[] Question: Do the main claims made in the abstract and introduction accurately reflect the paper's contributions and scope?
    \item[] Answer: \answerYes{}
    \item[] Justification: The abstract states specific numerical results (90.4\% reduction on Qwen2.5-1.5B, 72.0\% in the 55-prompt Pythia-410M attention replication, AUC 0.70, $2$--$66\times$ curvature gap, and improved circuit recovery on benchmark rankings) that are supported by Table~\ref{tab:master} and Figures~\ref{fig:diagnostic}, \ref{fig:curvature_gap}, and~\ref{fig:nonlocality}.  The cross-architecture claim is framed as practical generalization rather than a controlled causal comparison.

\item {\bf Limitations}
    \item[] Question: Does the paper discuss the limitations of the work performed by the authors?
    \item[] Answer: \answerYes{}
    \item[] Justification: We discussed the limitations of the work in the second paragraph of Section~\ref{sec:conclusion}, covering single-token perturbations, models up to 9B parameters, several auxiliary 20-prompt sweeps, and the fact that HVP improves component-level estimation rather than certifying full circuit completeness.

\item {\bf Theory assumptions and proofs}
    \item[] Question: For each theoretical result, does the paper provide the full set of assumptions and a complete proof?
    \item[] Answer: \answerYes{}
    \item[] Justification: Proposition~\ref{thm:main} states the local third-order smoothness assumption, the nonzero-denominator conditions, and the higher-order slack quantity $\alpha$, with the two-sided form recovered in the regime $\alpha < 1$. Full proofs are in Appendix~\ref{app:proofs}, Corollary~\ref{thm:hvp_error} gives the HVP residual bound, and Appendix~\ref{app:composition} labels multi-layer accumulation as a heuristic rather than as a formal theorem.

\item {\bf Experimental result reproducibility}
    \item[] Question: Does the paper fully disclose all the information needed to reproduce the main experimental results of the paper to the extent that it affects the main claims and/or conclusions of the paper?
    \item[] Answer: \answerYes{}
    \item[] Justification: Appendix~\ref{app:setup} provides model names, hook points, perturbation protocol, prompt count, bootstrap procedure, AtP*/GIM implementation details, and wall-clock timing. Code will be released.

\item {\bf Open access to data and code}
    \item[] Question: Does the paper provide open access to the data and code?
    \item[] Answer: \answerYes{}
    \item[] Justification: We provided the link to the code repo to reproduce the experimental results.

\item {\bf Experimental setting/details}
    \item[] Question: Does the paper specify all the training and test details?
    \item[] Answer: \answerYes{}
    \item[] Justification: Full details in Appendix~\ref{app:setup} including model versions, prompt sampling, perturbation protocol, wall-clock timing, and compute resources.

\item {\bf Experiment statistical significance}
    \item[] Question: Does the paper report error bars suitably and correctly?
    \item[] Answer: \answerYes{}
    \item[] Justification: The paper reports bootstrap 95\% confidence intervals from 1,000 prompt-level resamples for the model-sweep tables. The benchmark-specific reporting conventions, including point-estimate ranking metrics for IOI and Greater-Than, are described in Appendix~\ref{app:setup}.

\item {\bf Experiments compute resources}
    \item[] Question: For each experiment, does the paper provide sufficient information on the computer resources needed to reproduce the experiments?
    \item[] Answer: \answerYes{}
    \item[] Justification: Appendix~\ref{app:wallclock} specifies GPU type (NVIDIA L40S), wall-clock times per model, and total compute (approximately 25 hours including the scaled Qwen and Greater-Than runs).

\item {\bf Code of ethics}
    \item[] Question: Does the research conducted in the paper conform, in every respect, with the NeurIPS Code of Ethics?
    \item[] Answer: \answerYes{}
    \item[] Justification: This work improves interpretability methods, contributing to AI safety and transparency. No human subjects or sensitive data are involved.

\item {\bf Broader impacts}
    \item[] Question: Does the paper discuss both potential positive societal impacts and negative societal impacts of the work?
    \item[] Answer: \answerNA{}
    \item[] Justification: This is a methodological contribution to interpretability with no direct negative societal impacts.

\item {\bf Safeguards}
    \item[] Question: Does the paper describe safeguards that have been put in place for responsible disclosure of data with personal or offensive content?
    \item[] Answer: \answerNA{}
    \item[] Justification: No personal or offensive content is involved.

\item {\bf Licenses for existing assets}
    \item[] Question: Are the creators or original owners of assets (e.g., code, data, models) used in the paper properly credited?
    \item[] Answer: \answerYes{}
    \item[] Justification: Pythia~\citep{biderman2023pythia}, GPT-2~\citep{radford2019gpt2}, Qwen2.5~\citep{qwen2.5}, Gemma~2~\citep{gemma2024}, Llama~3.1~\citep{llama3_2024}, and TransformerLens~\citep{nanda2022transformerlens} are cited. The benchmark setups we build on are also cited in the main text.

\item {\bf New assets}
    \item[] Question: Are new assets introduced in the paper well documented and is the documentation provided alongside the assets?
    \item[] Answer: \answerYes{}
    \item[] Justification: Code release will include documentation and usage examples.

\item {\bf Crowdsourcing and research with human subjects}
    \item[] Question: For crowdsourcing experiments and research with human subjects, does the paper include the full text of instructions given to participants?
    \item[] Answer: \answerNA{}
    \item[] Justification: No human subjects are involved.

\item {\bf Institutional Review Board (IRB) approvals or equivalent for research with human subjects}
    \item[] Question: Does the paper describe potential risks incurred by study participants?
    \item[] Answer: \answerNA{}
    \item[] Justification: No human subjects are involved.

\item {\bf Declaration of LLM usage}
    \item[] Question: Does the paper describe the usage of LLMs if it is an important, original, or non-standard component of the core methods?
    \item[] Answer: \answerNA{}
    \item[] Justification: LLMs are not used as a component of the core method. The paper studies attribution methods applied \emph{to} transformer models; no LLM is used in implementing the proposed method itself.

\end{enumerate}

\end{document}